\documentclass{article}
\usepackage{arxiv}
\usepackage{balance}
\usepackage{graphicx}
\usepackage{amsmath}
\usepackage{amssymb}
\usepackage{amsthm}
\usepackage[noabbrev]{cleveref}
\usepackage{float}
\usepackage{forest}
\usepackage{mathtools}
\usepackage{amsfonts}
\usepackage{tikz}
\usepackage{tabularx}
\usepackage{multirow}

\newcommand{\ADL}{\ensuremath{\mathrm{ADL}}}
\newcommand{\ALC}{\ensuremath{\mathcal{ALC}}}
\newcommand{\concepts}{\ensuremath{\mathrm{X}}}
\newcommand{\roles}{\ensuremath{\mathrm{R}}}
\newcommand{\names}{\ensuremath{\mathrm{N}}}
\newcommand{\id}{\ensuremath{\mathtt{id}}}
\newcommand{\abook}{\ensuremath{\mathcal{A}}}
\newcommand{\tbook}{\ensuremath{\mathcal{T}}}
\newcommand{\akb}{\ensuremath{\mathcal{K}}}

\newcommand{\assert}[1]{\ensuremath{:^{#1}}}

\newcommand{\ite}[3]{\ensuremath{\left(#1\mathbf{?}#2\mathbf{:}#3\right)}} 
\newcommand{\cond}[3]{\ensuremath{[#3]\left(#1\! \mid\! #2\right)}}
\newcommand{\idcond}[2]{\ensuremath{(#1\! \mid\! #2)}}

\newcommand{\expect}{E}

\newcommand{\nmlt}{\ensuremath{\preceq}}
\newcommand{\eala}{\ensuremath{\approx}}

\newcommand{\interp}{\ensuremath{\mathcal{B}}}
\newcommand{\finterp}{\ensuremath{\mathcal{B}^f}}
\newcommand{\lang}[1]{\ensuremath{\mathcal{L}_#1}}
\newtheorem{theorem}{Theorem}[section]
\newtheorem{lemma}[theorem]{Lemma}

\newtheorem{definition}[theorem]{Definition}

\renewenvironment{proof}{\paragraph{Proof:}}{\hfill$\square$}

\title{Aleatoric Description Logic for Probabilistic Reasoning (Long Version)}

\author{Tim French\\
The University of Western Australia\\
{\tt tim.french@uwa.edu.au}
\and
Thomas Smoker\\
The University of Western Australia\\
{\tt thomas.smoker@uwa.edu.au}}

\begin{document}


\maketitle

\begin{abstract}
  Description logics are a powerful tool for describing ontological knowledge bases. 
  That is, they give a factual account of the world in terms of individuals, concepts and relations. 
  In the presence of uncertainty, such factual accounts are not feasible, and a subjective or epistemic approach is required. 
  Aleatoric description logic models uncertainty in the world as aleatoric events, 
  by the roll of the dice, where an agent has subjective beliefs about the bias of these dice. 
  This provides a subjective Bayesian description logic, 
  where propositions and relations are assigned probabilities according to what a rational agent would bet, 
  given a configuration of possible individuals and dice. 
  Aleatoric description logic is shown to generalise the description logic ALC, 
  and can be seen to describe a probability space of interpretations of a restriction of ALC where all roles are functions. 
  Several computational problems are considered and model-checking and consistency checking algorithms are presented.
  Finally, aleatoric description logic is shown to be able to model learning, 
  where agents are able to condition their beliefs on the bias of dice according to observations.
\end{abstract}


\section{Introduction}
\label{sect:introduction}
Description logics \cite{baader2003description} give a formal foundation for ontological reasoning: 
reasoning about factual aspects of the world. 
However, many reasoning tasks are performed in the presence of incomplete or uncertain information,
so a reasoner must apply some kind of belief model to approximate the true state of the world.
This work investigates the application of description logics to describing uncertain and incomplete concepts, 
following the recent development of aleatoric modal logic \cite{icla}. 
The term aleatoric has its roots in the Latin {\em aleator}, meaning dice player. 
and it is this origin that motivates this work. 
Concepts are not simply described as matters of fact, 
but can be more generally described as {\em reasonable bets}.
While a person may definitely have a virus or not, 
a virus test kit with 95\% accuracy is effectively a role of a dice, 
and often this aleatoric information is the only information available. 
In the medical domain observed symptoms and anatomical structure may be considered as fact,
but diagnosis and prognosis have an inherent degree of uncertainty.
Therefore concepts may be  modelled using probabilities corresponding to
a rational bet that the concept holds true, 
along the lines of the Dutch book argument of Ramsey \cite{ramsey1926truth} 
and de Finetti \cite{deFinetti}.

The fundamental assumption of this work is that an agent models the world
{\em aleatorically}, where events correspond to the roll of dice, 
and the bias of the dice is treated {\em epistemically}. 
That is, the agent has prior assumptions about the bias of the dice, 
and may refine these assumptions through observing the world.

Aleatoric description logic aims to model reasoning in 
uncertain and subjective knowledge settings \cite{halpern:2003,josang2016subjective}. 
However, aleatoric description logic only offers a useful approximation of subjective reasoning, 
rather than an ontic representation of reasoning that an actual agent may perform.
Imagine a rational agent that maintains an incomplete model of the world in their mind. 
When the agent is asked to make a judgement, 
they simulate a complete representation of the world by sampling the pertinent unknowns of the model, 
applying likelihoods coherent with their past experience. 
The judgement that is ``most likely'' becomes the position of the agent. 
There is a great variation in the way such models could be formed, and sampled against. 
Aleatoric description logic takes a pragmatic approach that assumes these models may be represented by dice based games 
(often refer to as role playing games, such as Advanced Dungeons and Dragons \cite{gygax1989}), 
where intricate sets of dice determine the unknowns of the model.  

Aleatoric description logic takes an approach
where the probabilistic and logical aspects of the knowledge base
are completely unified, rather than several other approaches where 
these are independent facets of the knowledge base 
\cite{ceylan,riguzzi2015probabilistic,lukasiewicz}.
Therefore all concept and roles are represented by ``dice rolls'' 
corresponding to an agent’s beliefs on the likely configuration of the world. 

An advantage of this ``probability first'' approach is that aleatoric description logic
is naturally able to model learning via Bayesian conditioning over complex observations 
(i.e. logical formula).
Aleatoric modal logic is introduced in \cite{icla}, 
where the semantics are presented along with a proof theoretic calculus. 
This paper extends that syntax and semantics to an aleatoric description logic, 
following as an analogue of the correspondence between 
description logics and modal logics \cite{baader2003description}.
Additionally: 
it will be shown how the aleatoric semantics are actually a probability space of functional models; 
knowledge base semantics are presented; 
satisfiability of acyclic axiom schemas is shown to be in PSPACE; 
and finally a learning framework will be presented where beliefs may be 
updated according to new observations, via Bayesian conditioning.

\section{Propositional description logic}\label{sect:ALC}

Description logic gives mathematical description of that which is (i.e. things that exist).
The logical formalisation allows us to determine when two concepts are equivalent, when one concept subsumes another, and when one concept can be extended to include another.

There are many variations of description logic with different expressivity and reasoning complexity \footnote{{\tt http://www.cs.man.ac.uk/~ezolin/dl/}}, 
but we will focus on the general purpose framework \ALC. 

The syntax of complex concepts in \ALC\ is is given by the following recursion:
$$ C\ ::=\ \top\ |\ A\ |\ C\sqcap C\ |\ \lnot C\ |\ \exists \rho.C$$
where $A\in \concepts$ is some atomic concept, and $\rho\in \roles$ is a role.
This syntax enables the expression of complex concepts built from atomic concepts.

\begin{definition}\label{def:ALCSemantics}
An {\em interpretation} of \ALC\ is a tuple $\mathcal{I} = (I, c, r)$ where
\begin{itemize}
  \item $I$ is a set of {\em individuals},
  \item $c:\concepts\longrightarrow\wp(I)$,
  \item $r:\roles\longrightarrow\wp(I\times I)$,
\end{itemize}
The semantics of \ALC\ are given with respect to a state $i\in I$ (the pointed interpretation ($\mathcal{I}_i$)) so that:
$\mathcal{I}_i\models \top$ always; $\mathcal{I}_i\models A$ iff $i\in c(A)$;
$\mathcal{I}_i\models C\sqcap D$ iff $\mathcal{I}_i\models C$ and $\mathcal{I}_i\models D$;
$\mathcal{I}_i\models \lnot C$ iff $\mathcal{I}_i\not\models C$; and
$\mathcal{I}_i\models \exists \rho. C$ iff for some $j$ such that $(i,j)\in r(\rho)$, $\mathcal{I}_j\models C$.
\end{definition}

The logic $\ALC_1$ is a particular semantic restriction of \ALC\ where the roles relation is constrained to be functional:
i.e. $\forall i\in I,\ \forall \rho\in\roles,\ |\{j\ |\ (i,j)\in r(\rho)\}|=1$.
The logic $\ALC_1$ is very basic and its concept satisfiability problem has a linear reduction to propositional logic. 
However, it provides an important foundation for the expressive relationship between $\ALC$ and $\ADL$. 
\begin{lemma}
  There is a computable linear transformation $\pi$ that maps $\ALC_1$ to propositional logic,
  and a computable linear transformation $\tau$ that maps interpretations of $\ALC_1$ 
  to interpretations of propositional logic such that for all pointed interpretations 
  $\mathcal{I}_i$ and all $\ALC$ formulas, $C$, $\mathcal{I}_i\models C$ if and only if $\mathcal{I}_i^\tau\models C^\pi$.
\end{lemma}
\begin{proof}
  For every word $w\in\roles^*$, and for every concept $X\in\concepts$, let $X_w$ be a propositional atom.
  For every word $w\in\roles^*$, we define a recursive transformation, 
  $(\dot)^w$, from $\ALC_1$ formulas to propositional formulas as follows:
  $$\top^w = \top\quad(C\sqcap D)^w = C^w\land D^w\quad (\lnot C)^w = \lnot D^w\quad\exists\rho. C = C^{\rho w}.$$
  Given a pointed $\ALC_1$ interpretation $\mathcal{I}_i = (I,c,r)$, 
  $\mathcal{I}_i^\tau$ is a subset of propositional atoms from $\{X_w\ |\ w\in\roles^*,\ X\in\concepts\}$, 
  defined recursively as follows:
  $$\mathcal{I}_i^\tau = \{X_\epsilon\ |\ i\in c(X)\}\cup 
  \bigcup_{\rho\in\roles} \{X_{\rho w}\ |\ \exists j,\ (i,j)\in r(rho),\ X_\rho\in\mathcal{I}_j^\tau\}$$
  where $\epsilon$ is the empty word. 
  It is straightforward to show $\mathcal{I}_i\models C$ if and only if $\mathcal{I}_i^\tau\models C^\epsilon$
  by induction over the complexity of formulas.
\end{proof}

\section{Aleatoric description logic}\label{sect:ADL}

This section presents the core syntax and semantics for Aleatoric Description Logic (\ADL).
This is a generalisation of standard description logics, such as \ALC, 
in the same sense that complex arithmetic is a generalisation of real-valued arithmetic:
the true-false/0-1 values of description logics are extended to the closed interval $[0,1]$.

\subsection{Syntax}
The syntax of \ADL\ varies from that of \ALC\  in a number of ways:
there is a ternary operator, {\em if-then-else}, in place of the typical Boolean operators,
and a {\em marginalisation} operator in place of the normal role quantifiers. 
These operators add expressivity, but also better capture the aleatoric intuitions of the logic.

The syntax for \ADL\ is specified with respect to a set of {\em atomic concepts}, $\concepts$ and a set of {\em roles}, $\roles$:
$$ \alpha\ ::=\ \top\ |\ \bot\ |\ A\ |\ \ite{\alpha}{\alpha}{\alpha}\ |\ \cond{\alpha}{\alpha}{\rho} $$
where $A\in\concepts$ is an atomic concept, and $\rho\in\roles$ is a role. 
Let the set of $\ADL$ formulas generated by this syntax be $\lang{\ADL}$.
This syntax uses non-standard operators and the following terminology is used: 
$\top$ is {\em always}; 
$\bot$ is {\em never}; 
$A$ is some named concept that may hold for an individual;
$\ite{\alpha}{\beta}{\gamma}$ is {\em if }$\alpha$ {\em then } $\beta$, {\em else }$\gamma$ (the {\em conditional operator}); and 
$\cond{\alpha}{\beta}{\rho}$ is $\rho$ {\em is} $\alpha$ {\em given} $\beta$ (the {\em marginalisation operator}).

We also identify a special role $\id\in\roles$ referred to as {\em identity}, 
which essentially refers to different possibilities for the one individual,
and write \idcond{\alpha}{\beta} in place of \cond{\alpha}{\beta}{\id}.

In these semantics every thing is interpreted as a probability dependent only on the individual: 
$\top$ always has probability 1.0 and $\bot$ always has probability 0.0;
an atomic concept $A$ has some probability that is dependent only on the current individual;
$\ite{\alpha}{\beta}{\gamma}$ has the probability of $\beta$ given $\alpha$ or $\gamma$ given not $\alpha$; 
and $\cond{\alpha}{\beta}{\rho}$ is the probability of $\alpha$ given $\beta$ over the set of individuals in the probability distribution corresponding to $\rho$.

\subsection{Motivation}\label{subsec:motivation}
Prior to giving formal semantics, this section gives an intuition for the new operators.
Imagine an agent, referred to as the {\em aleator}, who treats everything as a gamble much like a game of dice. 
But as a shrewd gambler, the aleator seeks to understand the dice, and know the probabilities better than anyone else.
In the aleator's mental model of the world, every individual and every entity can be modelled by a bag of dice:

\smallskip
To see if entity $e$ satisfies concept $C$, the aleator selects the bag of dice labelled $e$, 
    and then finds the die labelled $C$ in that bag. Every face of that die is either assigned $\top$ or $\bot$.
    The aleator rolls the die, and if $\top$ labels the upmost face, then $e$ satisfies $C$, {\em this time}. 
    Everything is a probability, and the next time the aleator plays this game, the result may be different. 
    But it's the probability of winning (or coming up $\top$) the aleator cares about, not some ideal truth.
    This is a simple game (the game of $C$ at $e$), and the more complex expressions in the syntax can build more complex games.

\smallskip    
To see if entity $e$ satisfies the proposition $\ite{\alpha}{\beta}{\gamma}$, the aleator first plays the game of $\alpha$ at $e$.
    If it comes up $\top$, then the aleator plays the game of $\beta$ at $e$. Otherwise, the aleator plays the game of $\gamma$ at $e$,  
    and in either case, the result of the second ``game'' is the final result.

\smallskip    
To see if entity $e$ satisfies the proposition $\cond{\alpha}{\beta}{\rho}$, 
    the aleator finds the die labelled $\rho$ in the bag of dice labelled $e$.
    This die is different to the dice labelled by concepts: instead of faces assigned $\top$ and $\bot$, 
    the faces of this die are labelled with other entities.
    To evaluate the proposition the aleator rolls the $\rho$ die, and notes the entity $e'$ on the upmost face.
    They then play the game of $\beta$ at $e'$ (i.e. by taking the bag of dice labelled $e'$ and evaluating $\beta$).
    If this comes up $\top$, the aleator then plays the game of $\alpha$ at $e'$ and this becomes the final result.
    If the game of $\beta$ at $e'$ came up $\bot$, then the aleator restarts the game, 
    going back to the bag of $e$ dice and rolling the $\rho$ die. 
    This process continues for as long as necessary, and if there is never a result, it defaults to $\top$.

\bigskip
This process maps every formula to either $\top$ or $\bot$ with a determined probability.
Thus it gives an aleatoric model of an agent's beliefs.
However, we presume that the agent is uncertain as to the true bias of the dice. 
They are simply Bayesian priors that an agent may refine through observation. 

To put this in the context of the initial example, 
the aleator wants to know the chance they have been exposed to a virus,
given they were in contact with someone who had a fever.
The proposition the aleator wants to evaluate is as follows:
$$\ite{\textit{virus}}{\top}{\cond{\textit{infectious}}{\textit{fever}}{\textit{contact}}}$$
``If they were already (asymptomatically) infected, they remain infected. 
Otherwise, given a contact selected from the population of people who have a fever, 
what is the probability that the contact is infected with the virus''.

To evaluate this they roll a die, to see if they already have the virus (there may be a 1\% chance).
In the cases where they didn't already have the virus, 
the aleator rolls a die to select a random person from the population of contacts,  
and then checks the person has a fever by rolling a die that represents the likelihood the person has a fever.
This is repeated until a random febrile contact is selected. 
Then a final die is rolled to determine whether that person is infectious. 

Aleatoric description logic does not describe the true world, 
nor does it describe a typical rational agent.
But it does describe a certain kind of rational agent, 
who can tolerate uncertainty and incorporate new evidence into their 
belief model in a simple and mathematically elegant way.

\subsection{Probabilistic Semantics}\label{semantics}

The \ADL\ is interpreted over an {\em Aleatoric Belief Model},
that is based on the probability model of \cite{icla} and  
the probability structures defined in \cite{halpern:2003}. 
It directly corresponds to the mental model of the aleator, consisting of the bags of dice.

\newcommand{\PD}{\ensuremath{\mathrm{PD}}}
\begin{definition}\label{def:PD}
  Given a set $S$, we use the notation $\PD(S)$ to notate the set of {\em probability distributions} over $S$, 
  where $\mu\in \PD(S)$ implies:
  $\mu:S\longrightarrow[0,1]$; and $\sum_{s\in S} \mu(s) = 1$.
\end{definition}

We use the notion of an {\em aleatoric belief model} as an interpretation of \ADL.
\begin{definition}\label{def:interpretation}
  Given a set of atomic concepts $\concepts$, and a set of roles $\roles$, an {\em aleatoric belief model} 
  is specified by the tuple $\interp =(I, r, \ell)$, 
  where:
  \begin{itemize}
    \item $I$ is a set of possible individuals.
    \item $r:\roles\times I\longrightarrow \PD(I)$ assigns for each role $\rho\in\roles$ and each individual $i\in I$, 
      a probability distribution $r(\rho,i)$ over $I$.
      We will typically write $\rho(i,j)$ in place of $r(\rho,i)(j)$. 
    \item For the role $\id$, we include the additional constraint:
      for all $i,j,k\in I$, $\id(i,j)>0$ implies $\id(j,k)=\id(i,k)$.
    \item $\ell:I\times\concepts\longrightarrow[0,1]$ gives the likelihood, 
      $\ell(i,C)$ of an individual $i$ satisfying an atomic concept $C$. 
      We will write $C(i)$ in place of $\ell(C,i)$.
  \end{itemize}
  Given some $i\in I$, we let $\interp_i$ be referred to as a {\em pointed aleatoric belief model}. 
\end{definition}

\begin{definition}\label{def:semantics}
  Given an aleatoric belief model $\interp=(I,r,\ell)$, some $i\in I$, and some $\alpha\in\ADL$ we specify the 
  probability $\interp$ assigns $i$ satisfying $\alpha$, $\interp_i(\alpha)$, recursively.
  We use the abbreviation, where $\rho\in\roles$: 
  $E^\rho_i\alpha = \sum_{j\in I}\rho(i,j)\interp_j(\alpha)$.
  Then:
  \begin{eqnarray*}
    \interp_i(\bot) = 0 &&
    \interp_i(\top) = 1\qquad
    \interp_i(C) = C(i)\\
    \interp_i(\ite{\alpha}{\beta}{\gamma}) &=& \interp_i(\alpha).\interp_i(\beta)+(1-\interp_i(\alpha)).\interp_i(\gamma)\\
    \interp_i(\cond{\alpha}{\beta}{\rho}) &=& \frac{\sum_{j\in I}\rho(i,j)\interp_j(\alpha)\interp_j(\beta)}{E^\rho_i\beta}\ \textrm{if}\ E^\rho_i\beta >0\\
    \interp_i(\cond{\alpha}{\beta}{\rho}) &=& 1,\ \textrm{if}\ E^\rho_i\beta=0 
  \end{eqnarray*}
\end{definition}

These semantics match the motivation of the aleator and the bag of dice. 
The probability of the game of $C$ at $i$ coming up $\top$ is $C(i)$.
When evaluating $\ite{\alpha}{\beta}{\gamma}$ at $i$, 
note the probabilities for the games of $\alpha$, $\beta$ and $\gamma$ at $i$ are all independent.
Therefore the values of the probabilities may be multiplied together: 
the probability of $\alpha$ and $\beta$ coming up $\top$ at $i$ is $\interp_i(\alpha).\interp_i(\beta)$, 
the probability of $\alpha$ coming up $\bot$ and $\gamma$ coming up $\top$ at $i$ is $(1-\interp_i(\alpha)).\interp_i(\gamma)$, 
and since $\alpha$ coming up $\top$ and $\alpha$ coming up $\bot$ are mutually exclusive 
(the game of $\alpha$ is only played once), these probabilities may be added together.
Finally, when evaluating $\cond{\alpha}{\beta}{\rho}$ we are simply taking the expectation of $\alpha$ over the distribution for $\rho$,
marginalised by the expectation of $\beta$ over the distribution for $\rho$.

An important property of these semantics is the weak independence assumption:
{\em All formulas of \ADL\ are contingent only on the individual at which they are evaluated}.
This means that two formulas evaluated at the same individual may be viewed as independent probabilistic events. 

Table~\ref{tab:abbreviations} gives a set of abbreviations familiar in the context of description logics.
\begin{table}
  \begin{center}
  \caption{\small Some abbreviations of operators in $\ADL$.}\label{tab:abbreviations}
    \begin{tabular}{|c|c|}
      \hline
      \begin{tabular}{|c|c|c|}
         \hline
         {\bf term} &{\bf formula} & {\bf interpretation}\\
         \hline
         $\alpha\sqcap\beta$ & $\ite{\alpha}{\beta}{\bot}$  & $\interp_i(\alpha).\interp_i(\beta)$\\
         \hline
         $\alpha\sqcup\beta$ & $\ite{\alpha}{\top}{\beta}$ & $\interp_i(\alpha) + \interp_i(\beta) - \interp_i(\alpha).\interp_i(\beta)$ \\
         \hline
         $\lnot\alpha$ & $\ite{\alpha}{\bot}{\top}$ & $1 - \interp_i(\alpha)$\\
         \hline
         $\alpha\Rightarrow\beta$ & $\ite{\alpha}{\beta}{\top}$ & $1 - \interp_i(\alpha)+\interp_i(\alpha).\interp_i(\beta)$\\
         \hline
         $\expect_\rho\alpha $&$ \cond{\alpha}{\top}{\rho}$ & $\sum_{j\in I}\rho(i,j).\interp_j(\alpha)$\\
         \hline
         $\exists\rho.\alpha $&$ \lnot\cond{\bot}{\alpha}{\rho}$ & $1$ if $\expect_\rho\alpha\neq 0$,  $0$ otherwise.\\  
         \hline
       \end{tabular}
       &    
       $
       \alpha^{\frac{n}{m}}= \left\{\begin{array}{ll} 
         1 & \textrm{ if }n=0\\ 
         0 &\textrm{ if } m<n\\ 
         \ite{\alpha}{\alpha^{\frac{n-1}{m-1}}}{\alpha^{\frac{n}{m}}} &\textrm{ if }n<m
       \end{array}\right.
       $\\
    \hline
  \end{tabular}
  \end{center}
\end{table}

The abbreviation in the right column of the table corresponds to a process of repeated sampling: 
Where $n,m\in\omega$, $\alpha^\frac{n}{m}$ corresponds to the likelihood of $\alpha$ being sampled at least $n$ times out of $m$. 
(A similar abbreviation can be defined for $\alpha$ coming up $\top$ \textit{exactly} $n$ times out of $m$.)
Note that this does not describe a probability or frequency, but an event. 
So $\alpha^{\frac{4}{5}}$ does not mean $\alpha$ is sampled at least 80\% of the time. 
Instead it describes the event of $\alpha$ being sampled $4$ times out of $5$, 
which would be quite likely (0.88) if $\alpha$ had probability 0.8, 
and unlikely (0.19) if $\alpha$ had probability 0.5.
This formalism can encode degrees of belief in an elegant way.
If an agent were to perform an action only if they believed $\alpha$ very strongly, 
one might set $\alpha^{\frac{9}{10}}$ as a precondition for the action, 
and if an agent were informed of a proposition $\beta$ by another agent 
who is considered unreliable,
they may update their belief base with the proposition $\beta^{\frac{2}{3}}$.

These operators may not appear logical: $\sqcap$ is not idempotent, 
and appears similar to the product t-norm of fuzzy logic \cite{zadeh}.
However, Section~\ref{sect:express} shows that these new operators
are inherently probabilistic and represent the process of 
reasoning over a probability space of description logic models.
Furthermore, restricting the concept probabilities to be 0 or 1, 
it can be seen that the semantic interpretation of 
$\sqcap$, $\lnot$ and $\exists\rho$ agrees with the 
standard description logic semantics, so classical description logic
can be seen as a special case of aleatoric description logic.

\subsection{Example}\label{sbsect:example}

For example, suppose we have three agents: {\tt Hector}, {\tt Igor} and {\tt Julia}. 
They each may have a virus ({\it V}), or not, and they also may have a fever ({\it F}), 
whether they have the virus or not. 
For each agent, we suppose that there are two possible individuals ({\bf PI}), 
one with the virus (e.g. {\tt Hector1}) and one without (e.g. {\tt Hector0}). 
For each possible individual, there is probability of them having a fever, 
which is naturally higher for possible individuals with the virus.
Each agent will occasionally come into contact with another agent, 
and the identity of this agent is described by the probability distribution {\tt contact}.
Finally, for each possible individual there is the probability of them being the actual agent ({\tt id}).


We can calculate the probability of an agent being newly exposed to the virus: 
$$\expect(\lnot V\sqcap \cond{V}{F}{\mathtt c})\nmlt\mathtt{exp}$$

An interpretation can be given this scenario. 
For each agent, we suppose that there are two possible individuals ({\bf PI}), 
one with the virus (e.g. {\tt Hector1}) and one without (e.g. {\tt Hector0}).
Note that the weighted probabilities of these agents satisfy the constraints of the A-Book, \abook. 

The probabilities for this scenario are given in Table~\ref{tab:ex-init-prob}, 
and a graphical representation is given in Figure~\ref{fig:example}. 
\newcommand{\hO}{\ensuremath{\mathtt{H_0}}}
\newcommand{\hI}{\ensuremath{\mathtt{H_1}}}
\newcommand{\iO}{\ensuremath{\mathtt{I_0}}}
\newcommand{\iI}{\ensuremath{\mathtt{I_1}}}
\newcommand{\jO}{\ensuremath{\mathtt{J_0}}}
\newcommand{\jI}{\ensuremath{\mathtt{J_1}}}
\begin{table}[ht]
\centering
    \caption{\small Initial probabilities for agent, contacts, virus and symptoms}\label{tab:ex-init-prob}
    \scalebox{1.0}{
     \begin{tabular}{|c|c|c|c|c|c|c|c|c|c|}\hline
        {\bf PI} & {\tt id} &{\it V} & {\it F} & \hO & \hI & \iO & \iI & \jO & \jI\\\hline
        {\tt Hector}$_0$ & 0.9 & 0.0 & 0.1 & 0.0 & 0.0 & 0.15 & 0.15 & 0.21 & 0.49\\\hline
        {\tt Hector}$_1$ & 0.1 & 1.0 & 0.6 & 0.0 & 0.0 & 0.15 & 0.15 & 0.21 & 0.49\\\hline
        {\tt Igor}$_0$ & 0.5 & 0.0 & 0.3 & 0.04 & 0.36 & 0.0 & 0.0 & 0.18 & 0.42\\\hline
        {\tt Igor}$_1$ & 0.5 & 1.0 & 0.8 & 0.04 & 0.36 & 0.0 & 0.0 & 0.18 & 0.42\\\hline
        {\tt Julia}$_0$ & 0.3 & 0.0 & 0.2 & 0.04 & 0.36 & 0.3 & 0.3 & 0.0 & 0.0\\\hline
        {\tt Julia}$_1$ & 0.7 & 1.0 & 0.9 & 0.04 & 0.36 & 0.3 & 0.3 & 0.0 & 0.0\\\hline
    \end{tabular}
  }
\end{table}
\begin{figure}
\centering
  \caption{\small A graphical example of the virus transmission scenario.}\label{fig:example}
    \scalebox{1.0}{
    \begin{tikzpicture}
      \tikzset{every node/.style={inner sep=0pt}}
      \draw (2,4.5) node[circle,draw](h0) {\small{ $\begin{array}{c}\hO\\{\mathit V}:0.0\\{\mathit F}:0.1\end{array}$ }};
      \draw (2,1.5) node[circle,draw](h1) {\small{ $\begin{array}{c}\hI\\{\mathit V}:1.0\\{\mathit F}:0.6\end{array}$ }};
      \draw (6,6) node[circle,draw](i0) {\small{ $\begin{array}{c}\iO\\{\mathit V}:0.0\\{\mathit F}:0.3\end{array}$ }};
      \draw (8,4) node[circle,draw](i1) {\small{ $\begin{array}{c}\iI\\{\mathit V}:1.0\\{\mathit F}:0.8\end{array}$ }};
      \draw (8,2) node[circle,draw](j0) {\small{ $\begin{array}{c}\jO\\{\mathit V}:0.0\\{\mathit F}:0.2\end{array}$ }};
      \draw (6,0) node[circle,draw](j1) {\small{ $\begin{array}{c}\jI\\{\mathit V}:1.0\\{\mathit F}:0.9\end{array}$ }};
      \draw[dashed] (h0) -- (h1) node[midway, right](pt-h) {$\mathtt{id}$} node[near start, left] {0.1} node[near end,left] {0.9};
      \draw[dashed] (i0) -- (i1) node[midway, below left](pt-i) {$\mathtt{id}$} node[pos=0.6, right] {0.5} node[pos=0.05, right] {0.5};
      \draw[dashed] (j0) -- (j1) node[midway, above left](pt-j) {$\mathtt{id}$} node[pos=0.4, right] {0.3} node[pos=0.95,right] {0.7};

      \draw[thick,<->] (pt-h) -- (pt-i) node[midway, below] {$\mathtt{c}$} node[near start, above] {0.4} node[near end,above] {0.3};
      \draw[thick,<->] (pt-i) -- (pt-j) node[midway, right] {$\mathtt{c}$} node[near start, left] {0.6} node[near end,left] {0.6};
      \draw[thick,<->] (pt-j) -- (pt-h) node[midway, above] {$\mathtt{c}$} node[near start, below] {0.7} node[near end,below] {0.4};
    \end{tikzpicture}
    }
\end{figure}

Interpreting this for {\tt Hector}, we see the probability {\tt Hector} was newly exposed to the virus is approximately 0.7.
The working for this is shown in Table~\ref{tab:calc}.
\begin{table}
  \begin{center}
 \caption{A calculation of the chance of Hector being newly exposed with the virus, 
  after a chance encounter with a person with a fever.}\label{tab:calc} 
  \begin{tabular}{|c|c|}\hline
    \iO & $F_{\iO} = 0.3,\quad(V\sqcap F)_{\iO} = 0.0,\quad \mathtt{c}(\hO,\iO) = \mathtt{c}(\hI,\iO)= .15$\\\hline
    \iI & $F_{\iI} = 0.8,\quad(V\sqcap F)_{\iO} = 0.8,\quad\mathtt{c}(\hO,\iO) = \mathtt{c}(\hI,\iO)= .15$\\\hline
    \jO & $F_{\jO} = 0.2,\quad(V\sqcap F)_{\jO} = 0.0,\quad\mathtt{c}(\hO,\jO) = 0.21,\quad \mathtt{c}(\hI,\jO)= .49$\\\hline
    \jI & $F_{\jI} = 0.9,\quad(V\sqcap F)_{\jO} = 0.9,\quad\mathtt{c}(\hO,\jO) = 0.21,\quad  \mathtt{c}(\hI,\jO)= .49$\\\hline
    \hO & $V_0 = 0.0,\quad\mathtt{id}_0 = 0.9,\quad\cond{V}{F}{\mathtt{c}} = 
      \frac{\sum_{x=\iO}^{\jI}\mathtt{c}(\hO,x)\cdot(V\sqcap F)_x}{\sum_{x = \iO}^{\jI}\mathtt{c}(\hO,x).F_x} = 0.78$\\\hline
   \hI & $V_1 = 1.0,\quad\mathtt{id}_1 = 0.1,\quad\cond{V}{F}{\mathtt{c}} = 
      \frac{\sum_{x=\iO}^{\jI}\mathtt{c}(\hI,x)\cdot(V\sqcap F)_x}{\sum_{x = \iO}^{\jI}\mathtt{c}(\hI,x) \cdot F_x} = 0.78$\\\hline
    $\mathtt{H}$ & $\expect(\lnot V\sqcap \cond{V}{F}{\mathtt c}) = \sum_{x=0}^1 (1-V_x) \cdot \mathtt{id_x}.\cond{V}{F}{\mathtt{c}} = 0.7$\\\hline
  \end{tabular}
  \end{center}
\end{table}

\section{Aleatoric Knowledge Base Semantics}

A description logic knowledge base is defined as {\ensuremath{\mathcal{K = \langle T, A \rangle}}}; 
where {\ensuremath{\mathcal{K}}} is a knowledge base; 
{\ensuremath{\mathcal{T}}}, a \textit{TBox} is a set of axioms on the properties of concepts, known as intensional assertions; 
and {\ensuremath{\mathcal{A}}}, an \textit{ABox}, is a set of axioms on the groundings of concepts, 
called extensional assertions \cite{calvanese}. 

The aleatoric belief models of Section~\ref{semantics} describe a complete interpretation for formulas of aleatoric description logic,
and so requires a notion of a closed world.
For the representations of an agent's subjective knowledge in an open world, 
it is only possible to record a subset of what the agent knows and believes,
and consider the class of all models that support this belief set.
Just as description logics are applied to knowledge bases consisting of assertional axioms (A-Boxes) and terminological axioms,
aleatoric description logics are applied to {\em aleatoric knowledge bases} 
consisting of assertional axioms (A-Books) and terminological axioms (T-Books).

An aleatoric knowledge base is defined over the same signature of 
atomic concepts $\concepts$, and roles $\roles$, including ${\mathtt{id}}$.
Additionally there is a set of {\em named individuals}, $\names$, 
which may be thought of as special concepts for grounding assertions and framing queries.
In line with the epistemic nature of these knowledge bases each named 
individual can be any one of a number of {\em possible individuals}, 
and the distribution of these possible individuals is represented by the role $\mathtt{id}$.

As with \ALC\ we have terminological axioms and assertional axioms.
{\em Aleatoric terminological axioms} or {\em T-Books} describe rules that are universally true for all individuals,
and thus provide a non-probabilistic intensional definition of the concepts and roles in the logic.
{\em Aleatoric assertional axioms} or {\em A-Books} describe subjective extensional information 
by listing the probabilities with which individuals satisfy given concepts and roles.
It is not the case that {\em T-Books} describe {\em concept inclusion} nor {\em subsumption} as \textit{TBoxes} do in \ALC, 
as these concepts do not make sense when considering a set of weakly independent propositions.
Instead {\em T-Books}, provide a means to {\em constrain strength of belief}. 
The semantics are given in Definition~\ref{def:nmlt}.

\begin{definition}\label{def:nmlt}
  The {\em aleatoric terminological axioms} have the form:
$$\begin{array}{cl} \alpha\nmlt\beta&\alpha\textrm{ is no more likely than }\beta\\
  \alpha\eala\beta &\alpha\textrm{ is exactly as likely as }\beta.
\end{array}
$$
Given an aleatoric belief model $\interp = (I, R, \ell)$, 
$\interp$ {\em satisfies} $\alpha\nmlt\beta$ (written $\interp\models\alpha\nmlt\beta$), 
if and only if for all $i\in I$, $\interp_i(\alpha)\leq\interp_i(\beta)$;
and $\interp$ {\em satisfies} $\alpha\eala\beta$ (written $\interp\models\alpha\eala\beta$) 
if and only if $\interp\models\alpha\nmlt\beta$ and $\interp\models\beta\nmlt\alpha$.
A {\em T-Book} is a set of aleatoric terminological axioms.
\end{definition}

These axioms place universal constraints on the likelihoods of aleatoric formulas being true.
For example we might include an axiom $\textit{first}\nmlt\textit{place}$, 
meaning coming first in a race is no more likely than placing (coming first, second or third).
Alternatively, we could define placing precisely as coming first, second or third, 
via the axiom $\textit{place}\eala\textit{first}\sqcup\textit{second}\sqcup\textit{third}$, 
and then $\textit{first}\nmlt\textit{place}$ is implicitly true.

\begin{definition}\label{def:abook}
  The {\em aleatoric assertional axiom} (or simply assertions) have the form: 
  \begin{itemize}
    \item $a\assert{p} \alpha$, where $a\in \names$, $p\in[0,1]$ and $\alpha\in\ADL$ asserts that 
      individual $a$ satisfies concept $\alpha$, with probability $p$.
    \item $(a,b)\assert{p}  \rho$, where $a, b\in \names$, $p\in[0,1]$ and $\rho\in \roles$ 
      asserts that individual $b$ satisfies the role $\rho$ for $a$ with probability $p$. 
  \end{itemize}
  An {\em a-book} $\mathcal{A}$ is a set of aleatoric assertional axioms, 
  and $\mathcal{A}$ is a {\em well-formed a-book} if for every $a\in N$, for every $\rho\in R$, 
  $\sum\{p\ |\ (a,b)\assert{p}\rho\} \leq 1$.
  An {\em A-Book}, $\abook$ is {\em simple} if for all aleatoric assertional axioms $\sigma\in\abook$ of the form
  $a\assert{p} \alpha$, it is the case that $\alpha$ is an atomic concept.
\end{definition}

While an A-Book is existentially quantified, T-Books are universally quantified and consequently a very powerful formalism.
Therefore, it is useful to consider a restriction on T-Books referred to as an {\em acyclic T-Book}.
\begin{definition}\label{def:t-book}
  A concept $C$ is an {\em atom} if $C\in\concepts\cup\{\top,\bot\}$ (i.e. $C$ is an atomic concept, always, or never).
  A terminological axiom is {\em simple} if it has the one of the forms 
  \begin{itemize}
    \item $C\eala \ite{D}{E}{F}$ where $C,\ D,\ E$ and $F$ are all atoms.
    \item $C\eala \cond{D}{E}{\rho}$ where $C,\ D$ and $E$ are all atoms.
  \end{itemize}
  A {\em simple T-Book}, $\tbook$, is a T-Book consisting only of simple terminological axioms.
  A simple T-Book, $\tbook$, is {\em acyclic} if there is no sequence of concepts $C_0,\hdots,C_n$ where:
  \begin{itemize}
    \item for all $i=1,\hdots,n$, either:
      \begin{itemize}
        \item there is some $C\eala\ite{D}{E}{F}\in\tbook$, where $C_{i},C_{i-1}\in\{C,D,E,F\}\cap\concepts$;
        \item there is some $C\eala\cond{D}{E}{\rho}\in\tbook$, where $C_{i},C_{i-1}\in\{C,D,E\}\cap\concepts$;
      \end{itemize}
    \item there is some $i$ where $C\eala\cond{D}{E}{\rho}\in\tbook$ and $C_i,C_{i-1}\in\{C,D,E\}\cap\concepts$;
    \item $C_0 = C_n$.
  \end{itemize}
\end{definition}

\newcommand{\faircoin}{\ensuremath{C_{\textit{fair}}}}
As a brief example of a useful terminological axiom the following axiom 
constrains the concept $\faircoin$ to be a fair coin: $\faircoin\eala\lnot\faircoin$.
Now for all individuals, one can assume that there is a concept, $\faircoin$, 
available that has precisely a half chance of coming up $\top$.

A simple T-Book and a simple A-Book make a simple aleatoric knowledge base.
\begin{definition}\label{def:akb}
  An {\em aleatoric knowledge base} $\akb = (\abook,\tbook)$ 
  is a pair consisting of a set of assertional axioms $\abook$ 
  and a set of terminological axioms $\tbook$.
  If $\abook$ is a simple A-Book and $\tbook$ is a simple T-Book, 
  the $\akb$ is a {\em simple aleatoric knowledge base}, and if $\tbook$ is also acyclic
  $\akb$ is an {\em acyclic simple knowledge base}.
\end{definition}

An aleatoric knowledge base describes a belief, or subjective position of an agent, 
that can correspond to a number of different interpretations. 
The interpretations {\em satisfy} the aleatoric knowledge base if they make all the axioms true.  

\begin{definition}\label{def:satisfies}
  Given an aleatoric knowledge base $\akb=(\abook,\tbook)$ over the signature ($\concepts$, $\roles$, $\names$), 
  and an aleatoric belief model $\interp = (I, r, \ell)$ over the signature ($\concepts\cup \names$, $\roles$) {\em satisfies} $\mathcal{K}$ iff:
\begin{itemize}
  \item For every $a\in\names$, for every $i\in I$, $a(i)\in\{0,1\}$ and for all $i,j\in I$, $a(i)=1$ and $\id(i,j)>0$ implies 
    $a(j)=1$. That is, the names are absolute concepts, and two possibilities for a single individual will share a name. 
  \item For every terminological axiom $\alpha\nmlt\beta\in\tbook$, for every $i\in I$, 
    it follows that $\interp_i(\alpha)\leq\interp_i(\beta)$.
  \item For every terminological axiom $\alpha\eala\beta\in\tbook$, for every $i\in I$, 
    it follows that $\interp_i(\alpha)=\interp_i(\beta)$.
  \item For every assertion $a\assert{p}\alpha\in\abook$, for every $i\in I$ with $a(i)=1$, it follows that $\interp_i(\expect\alpha)=p$.
  \item For every assertion $(a,b)\assert{p} \rho\in\abook$ for every $i\in I$ with $a(i)=1$, $\sum_{j\in I}\rho(i,j).b(j) = p$.
\end{itemize}
  We say that a knowledge base $\mathcal{K}$ is {\em consistent} if it is supported by at least one aleatoric belief model. 
\end{definition}

The following lemma is a useful simplification.
\begin{lemma}\label{lem:simple}
  Every aleatoric knowledge base $\akb = (\abook,\tbook)$ is equivalent to a simple aleatoric knowledge base, $\akb'$.
\end{lemma}
\begin{proof}
  For every terminological axiom $\tau\in\mathcal{T}$ introduce a fresh atomic concept $E_\tau$;
  for every non-atomic subformula $\alpha$ appearing in some terminological or assertional axioms $\tau\in\abook\cup\tbook$, 
  introduce a fresh atomic concept $C_\alpha$; 
  and for every atomic concept $\alpha\in\concepts$ appearing in some axiom $\tau\in\abook\cup\tbook$, 
  let $C_\alpha=\alpha$.  
  If $\alpha = \ite{\beta}{\gamma}{\delta}$, then let $\alpha^* = \ite{C_\beta}{C_\gamma}{C_\delta}$, 
  and if $\alpha = \cond{\beta}{\gamma}{\rho}$, $\alpha^* = \cond{C_\beta}{C_\gamma}{\rho}$.
  The T-Book, $\tbook'$ is then the set of axioms:
  $$\left\{\begin{array}{ll}
    C_\alpha\eala\alpha^* &|\ \alpha\textrm{ appears in }\tau\in\abook\cup\tbook\\
    C_\alpha\eala C_\beta &|\ \alpha\eala\beta\in\tbook\\
    C_\alpha\eala \ite{C_\beta}{E_\tau}{\bot} &|\ \tau=\alpha\nmlt\beta\in\mathcal{T}
  \end{array}\right\},$$
  the A-Book, \abook', is $$
  \left\{\begin{array}{ll}
    a\assert{p} C_\alpha & |\ a\assert{p}\alpha\in\abook\\
  (a,b)\assert{p}\rho & |\ (a,b)\assert{p}\rho\in\abook\end{array}\right\},$$
  and $\akb' = (\abook',\tbook')$. 
  It is easy to see that given any aleatoric belief model $\interp=(I,R,\ell)$ satisfying $\akb$, 
  a corresponding aleatoric belief model $\interp' = (I,R,\ell)$ may be defined to satisfy $\akb'$ 
  by setting, for all $i\in I$, for all $\alpha $ appearing in $\tau\in\abook\cup\tbook$, $\ell'(i,C_\alpha) = \interp_i(\alpha)$,
  and for all $i\in I$, for all $\tau = \alpha\nmlt\beta\in\mathcal{T}$ setting $\ell'(i,E_\tau) = \interp_i(\alpha)/\interp_i(\beta)$ 
  (this is guaranteed to be between 0 and 1, since $\interp$ satisfies $\alpha\nmlt\beta$).
\end{proof}

For an aleatoric knowledge the questions of interest are:
\begin{itemize}
  \item {\em Satisfiability:} Given an aleatoric knowledge base, $\akb = (\abook,\tbook)$, is it consistent.  
  \item {\em Concept Satisfiability:} Given an aleatoric knowledge base, $\akb = (\abook,\emptyset)$, is it consistent.  
  \item {\em Concept query:} Given an aleatoric knowledge base, $\akb = (\abook,\tbook)$, 
    what is the lower (upper) bound on the likelihood of a concept for some named individual.
\end{itemize}
The following subsection extend the example of Section~\ref{sbsect:example} to the aleatoric knowledge base semantics, 
and then the subsequent section will 
consider the complexity of answering these questions. 

\subsection{Example}\label{sbsect:akb-example}

The case from Section~\ref{sbsect:example} can now be expressed as an aleatoric knowledge base.
Now, rather than needing to assign a probability to every role and concept, 
only the propositions an agent has a genuine subjective position on are given.

For example, an aleatoric knowledge base could model that Hector is very likely not to have the virus;
Julia is likely to have the virus, Julia is very likely to have a fever and to it is likely that Hector came into contact with Julia.
Furthermore, a terminological axiom can specify the belief that a new exposure to the virus ($\mathtt{exp}$)
occurs if an agent did not already have the virus, but came into contact with some febrile person who did have the virus. 
Thus the aleatoric knowledge base $\mathcal{K} = (\{\mathcal{T}_1\},\{\mathcal{A}_1,\hdots,\mathcal{A}_4\}$ is:
$$
\begin{array}{ll}
  \mathcal{A}_1 & \mathtt{Hector}\assert{0.1} V\\ 
  \mathcal{A}_2 & \mathtt{Julia}\assert{0.7} V\\
  \mathcal{A}_3 & \mathtt{Julia}\assert{0.69} F\\
  \mathcal{A}_4 & (\mathtt{Hector},\mathtt{Julia})\assert{0.3} \mathtt{c}\\
  \mathcal{T}_1 & \expect(\lnot V\sqcap \cond{V}{F}{\mathtt c})\nmlt\mathtt{exp}
\end{array}
$$
This aleatoric knowledge base is satisfied by the interpretation presented in Section~\ref{sbsect:example}. 
However, many other interpretations would also satisfy $\mathcal{K}$. 
Igor, the other possible contact of Hector, does not appear in the knowledge base, 
so in the 70\% chance that Julia was not a contact of Hector, the actual contact is truly arbitrary. 
Therefore, the knowledge base semantics do not require one to take a position on propositions they have no information on (subjective or otherwise).
The aleatoric knowledge base can be used to answer queries. 
To determine if the knowledge base necessitates that there is a greater than 25\% chance of Hector being newly exposed to the virus,
the axiom $\mathtt{Hector}\assert{0.25}\mathtt{exp}$ can be inserted into the knowledge base, 
and consistency checking can be applied. This process is described in the following section. 



\subsection{Reasoning with Aleatoric Description Logic}\label{sbsect:akb-reason}

This section will consider computational properties of aleatoric description logic.
The particular questions considered are:
\begin{itemize}
  \item {\em Model Checking} Given an pointed belief model $B_i$ and some formula $\alpha$, what is the value of $B_i(\alpha)$ (the probability assigned to $\alpha$ by $B_i$).
  \item {\em Belief Set Consistency} Given a belief set $\mathcal{K} = (\mathcal{T},\mathcal{A})$, is there a belief model that agrees with the belief set on all axioms.
\end{itemize}

An aleatoric knowledge base may correspond to many aleatoric belief models, or possibly none.
The main question of interest is whether there is any interpretation that could possibly 
correspond to a given aleatoric knowledge base. 
However, by assigning flat priors to all unknown (or {\em ambivalent} concepts) 
one can define an interpretation and get a partial answer via model-checking. 

\begin{lemma}
  Given a pointed belief model $\interp_i$ consisting of $n$ possible individuals, and a formula $\alpha$ consisting of $m$ symbols, 
  the value $\interp_i(\alpha)$ can be computed in time $O(n^2m)$.
\end{lemma}
\begin{proof}
  This computation is done by applying the semantic definitions recursively. 
  All operations can be done in constant time ($O(1)$) except marginalisation which is $O(n)$.
  As operations need to be done for every possible individual, and there are at most $m$ operations,
  the complexity has an upper bound of $O(n^2m)$.
\end{proof}

To be able to perform inference based on an aleatoric knowledge base, we must first determine if it is consistent 
(i.e. agrees with at least one aleatoric belief model). 
Given a simple aleatoric knowledge base $\mathcal{K}= (\mathcal{T}, \mathcal{A})$, where $\mathcal{T}$ is an acyclic t-book, 
it is possible to determine if $\mathcal{K}$ is satisfiable with complexity PSPACE. 
The case for non-acyclic t-books is left to future work.

The process for the satisfiability theorem is to build a system of polynomial equalities and inequalities
corresponding to the axioms in $\akb$. 
This system of constraints is satisfiable if and only if $\akb$ is satisfiable.
The system of equalities corresponds closely to the semantic interpretation of the axioms,
over a sufficiently large set of individuals, that are constrained by linear inequalities.
Given an aleatoric knowledge base $\akb=(\tbook,\abook)$, 
let $\Lambda^\akb$ be a set of variables, 
let $\Phi^\akb$ be a set of polynomial equalities over $\Lambda^\akb$, 
and let $\Psi^\akb$ be a set of linear inequalities over $\Lambda^\akb$. 
These sets are constructed as follows.

\newcommand{\ol}[1]{\ensuremath{\overline{#1}}}
As $\tbook$ is simple and acyclic, it will only contain axioms of the form: $C\eala\ite{D}{E}{F}$, $C\eala\cond{D}{E}{\rho}$ and $C\eala D$, 
where $C$, $D$, and $E$ are all atomic. 
Define the relation $\cong\subseteq\concepts^2$ over the set of atomic concepts by: 
\begin{enumerate}
  \item for all $C in\concepts$, $C\cong C$,
  \item $C\cong D$ and $D\cong E$, implies $C\cong E$,
  \item if $C\eala\ite{D}{E}{F}\in\tbook$, then $X\cong Y$ for all $X,Y\in\{C,D,E,F\}\cap\concepts$.
  \item if $C\eala\cond{D}{E}{\rho}\in\tbook$, and $C,D\in\concepts$, then $D\cong E$, and $E\cong D$
  \item if $C\eala D\in\tbook$ and $C,D\in\concepts$, then $C\cong D$, and $D\cong C$.
\end{enumerate}
It is clear that $\cong$ is an equivalence relation, 
so for every atomic concept $C\in\concepts$ let $\overline{C}$ be the corresponding equivalence class.
Let $\overline{\concepts}$ be the set of equivalence classes, $\overline{C}$, where $C\in\concepts$, 
Given $\ol{C},\ol{D}\in\overline{\concepts}$, and $\rho\in\roles$,
let $\rho^\#(\ol{C},\ol{C})$ be the 
number of axioms $E\eala\cond{F}{G}{\rho}$ that appear in $\tbook$, where $E\in \ol{C}$ and $\{F,G\}\cap c\neq\emptyset$.
Let $c\Rightarrow d$ if and only if for some $\rho\in\roles$,
$\rho^\#(\ol{C},\ol{D})>0$, or if $\ol{C}\in\names$ and $\ol{D}\notin\names$.
Then $(\overline{\concepts},\Rightarrow)$ is a directed acyclic graph.

For each node of the graph, $\ol{C}\in\ol{\concepts}$, define:
$$\#\ol{C}  = \sum_{\begin{array}{c}\rho\in\roles\\ \ol{D} \in\ol{\concepts}\end{array}}\rho^\#(\ol{D},\ol{C})$$
  so $\#\ol{C}$ is the number of TBook axioms that contain a concept related to $ c $ in the scope of a marginalisation operator.

For each $\ol{C}\in\ol{\concepts}$, for each atomic concept $D\in\ol{C}$, 
assign $\#\ol{C}$ different variables $x_1^D,\hdots, x^D_{\#\ol{C}}\in\Lambda^\akb$,
and add the inequalities $\{0\leq x_i^D\leq 1\ |\ D\in\ol{C},\ i =1,\hdots,\#\ol{C}\}$ to $\Psi^\akb$. 
The following equations are added to $\Phi^\akb$, supposing that for all $i$, $x^\top_i = 1$ and $x^\bot_i = 0$:
\begin{itemize}
  \item For each axiom of the form $C\eala(D?E:F)\in\tbook$ add to $\Phi^\akb$ the equations
    $$x^C_i = x^D_i.x^E_i+(1-x^D_i).x^F_i,\ \textrm{ for }i = 1,\hdots\#\ol{C},$$
    noting that it must be the case $D, E, F\in c \cup\{\top,\bot\}$.
  \item For each axiom of the form $C\eala D\in\tbook$, add to $\Phi^\akb$ the equations
    $x^C_i = x^D_i$ for $i = 1,\hdots, \#\ol{C}$.
\end{itemize}

For every $\rho\in\roles$, for every pair $\ol{C},\ol{D}\in \ol{\concepts}$, 
add the variables $\{r^{\ol{C}\rho\ol{D}}_{ij}\ |\ i = 1,\hdots\#\ol{C},\ j = 1,\hdots,\#\ol{D}\}$ to $\Lambda^\akb$,
and add the inequalities $\{0\leq r^{\ol{C}\rho\ol{D}}_{ij}\leq 1\ |\ i = 1,\hdots\#\ol{C},\ j = 1,\hdots,\#\ol{D}\}$ to $\Psi^\akb$.
The following equations are added to $\Phi^\akb$:
\begin{itemize}
  \item For every pair $\ol{C},\ol{D}\in\ol{\concepts}$, for every $\rho\in\roles$, for each $i = 1,\hdots,\#\ol{C}$, 
    add to $\Phi^\akb$ the equations $\sum_{j = 1}^{\#\ol{D}} r^{\ol{C}\rho\ol{D}}_{ij} = 1$.
  \item For every axiom $C\eala\cond{D}{E}{\rho}\in\tbook$, for $i = 1,\hdots,\#\ol{C}$, 
    add to $\Phi^\akb$ the equation:
    $$ \sum_{j=1}^{\#\ol{D}}r^{\ol{C}\rho\ol{D}}_{ij}.x_j^E.x_i^C = \sum_{j=1}^{\# d }r^{\ol{C}\rho\ol{D}}.x^D_j.x^E_j,$$
    noting that $E\in\ol{D}\cup\{\top,\bot\}$.
  \item To account for the case where $C\eala\cond{D}{E}{\rho}\in\tbook$, and $E$ is false, add the variables
    $\{e^{\ol{C}\rho\ol{D}}_{i}\ | \i = 1,\hdots,\#\ol{C}\}$ to $\Lambda^\akb$, along with the inequalities:
    $\{0\leq e^{\ol{C}\rho\ol{D}}_{i}\ | \i = 1,\hdots,\#\ol{C}\}$ to $\Phi^\akb$. 
    Note, these are not constrained to be less than 1. Then, add to $\Phi^\akb$ the equalities:
    $$(\sum_{j=1}^{\#\ol{D}}r^{\ol{C}\rho\ol{D}}_{ij}.x_j^E).e^{\ol{C}\rho\ol{D}}_i+x_i^C = 1.$$
    Therefore if $\sum_{j=1}^{\#\ol{C}}r^{\ol{C}\rho\ol{D}}_{ij}.x_j^E=0$ then $x_i^C$ must equal 1.
  \item The $\id$ relation has additional constraints which are captured by adding to $\Phi^\akb$, 
    for all $\ol{C}\in\ol{\concepts}$, for every $i,j,k = 1,\hdots,\#\ol{C}$, 
    the equality $r^{\ol{C}\id\ol{C}}_{ij} = r^{\ol{C}\id\ol{C}}_{ik}$.
\end{itemize}

Where $D,E\in\ol{C}$, for each $i=1,\hdots\#\ol{C}$, 
the variables $x^D_i$ and $x^E_i$ describe the probability of the concepts $D$ and $E$ holding at the same individual.
To build a complete model it is necessary to describe a correspondence between different partitions of concepts, 
as they will be describing different concepts for a common individual.

For every $\ol{C},\ol{D}\in\ol{\concepts}$, 
add the variables $\{e^{\ol{C}\ol{D}}_{ij}\ |\ i = 1,\hdots,\#\ol{C},\ j = 1,\hdots,\#\ol{D}\}$ to $\Lambda^\akb$,
and add the inequalities $\{0\leq e^{\ol{C}\ol{D}}_{ij}\leq 1\ |\ i = 1,\hdots,\#\ol{C},\ j = 1,\hdots,\#\ol{D}\}$ to $\Psi^\akb$.
The following equalities are added to $\Phi^\akb$:
\begin{itemize}
  \item For all $\ol{C}\in\ol{\concepts}$, for all $i = 1,\hdots,\#\ol{C}$, add $e^{\ol{C}\ol{C}}_{ii} = 1$.
  \item For all $\ol{C},\ol{D} \in\ol{\concepts}$, for all $i = 1,\hdots,\#\ol{C}$, $\sum_{j=1}^{\# d }e^{\ol{C}\ol{D} }_{ij} = 1$.
  \item For all $\ol{C}, \ol{D},\ol{E} \in\ol{\concepts}$, for all $i = 1,\hdots,\#\ol{C}$, 
    for all $j = 1,\hdots,\#\ol{D}$, for all $\rho\in\roles$,
    $$e^{\ol{C}\ol{D}}_{ij}.\sum_{k=1}^{\#\ol{E}}r^{\ol{D}\rho\ol{E}}_{jk} = 
    e^{\ol{D}\ol{C}}_{ji}.\sum_{k=1}^{\#\ol{E}}r^{\ol{C}\rho\ol{E}}_{ik}.$$
  \item For all $\ol{C}, \ol{D},\ol{E} \in\ol{\concepts}$, for all $i = 1,\hdots,\#\ol{C}$,
    $$\sum_{k=1}^{\#\ol{E}}\sum_{j=1}^{\#\ol{D}}r^{\ol{C}\rho\ol{D}}_{ij}.e^{\ol{D}\ol{E}}_{jk} = 
    \sum_{j=1}^{\#\ol{D}}\sum_{k=1}^{\#\ol{E}}r^{\ol{C}\rho\ol{E}}_{ik}.e^{\ol{E}\ol{D}}_{kj}.$$
\end{itemize}

It is left to represent the axioms in $\abook$.
For every $a\in\names$, let $\# a$ be the number of axioms in $\abook$ that $a$ appears in,
and for every $C\in\concepts$, for every $i = 1,\hdots,\# a$, add a fresh variable $n^{aC}_i$ to $\Lambda^\akb$.
For every $a\in\names$, for every $C\in\concepts$, for $i = 1,\hdots,\#a$ add the inequalities $0\leq n^{ac}_i\leq 1$
to $\Psi^\akb$.

For every $a\in\names$, for every $\rho\in\roles$, for every $\ol{C}\in\ol{\concepts}$, for $i = 1,\hdots,\# a$, 
for $j = 1,\hdots,\#\ol{C}$, 
add a variable $r^{a\rho C}_{ij}$ to $\Lambda^\akb$, and add the inequalities $0\leq r^{a\rho C}_{ij}\leq 1$ to $\Psi^\akb$.
For every $a,b\in\names$ for $i = 1,\hdots,\# a$, for $j = 1,\hdots,\# b$, add a variable $r^{a\id b}_{ij}$ to $\Lambda^\akb$,
add the inequalities $0\leq r^{a\id a}_{ij}\leq 1$ to $\Psi^\akb$, 
and for every $i,j,k=1,\hdots,\# a$, add the equality $r^{a\id a}_{ij} = r^{a\id a}_{ik}$ to $\Phi^\akb$.
\begin{itemize}
  \item For each axiom $a\assert{x} C$ in the simple A-Book $\abook$, for every $i = 1,\hdots, \# a$, 
    add the equality $\sum_{i = 1}^{\# a} r^{a\id a}_{ij}n^{ac}_j = x$ to $\Phi^\akb$.
  \item For each axiom $(a,b)\assert{x} \rho$, for every $i = 1,\hdots,\# a$ add the equality 
    $$\sum_{j = 1}^{\# a}r^{a\id a}_{ij}\sum_{k=1}^{\# b}r^{a\rho b}_{jk} = x$$
    to $\Phi^\akb$.
  \item For every $\ol{C}\in\ol{\concepts}$, for every $\rho\in\roles$, for each $i = 1,\hdots,\# a$, 
    add to $\Phi^\akb$ the equations 
    $$\sum_{j = 1}^{\#\ol{C}} r^{a\rho\ol{C}}_{ij}+\sum_{b\in \names}\sum_{j=1}^{\# b}r^{a\rho b}_{ij} = 1.$$
  \item For every axiom $C\eala\cond{D}{E}{\rho}\in\tbook$, for $i = 1,\hdots,\# a$, 
    add to $\Phi^\akb$ the equations:
    $$ (\sum_{j=1}^{\#\ol{D}}r^{a\rho\ol{D}}_{ij}.x_j^E + \sum_{b\in\names}\sum_{j=1}^{\# b} r^{a\rho b}n^{bE}_j).n_i^{aC} 
    = \sum_{j=1}^{\#\ol{D}}r^{a\rho\ol{D}}.x^D_j.x^E_j+\sum_{b\in\names}\sum_{j=1}^{\# b} r^{a\rho b}n^{bD}_j.n^{bE}_j.$$
  \item To account for the case where $C\eala\cond{D}{E}{\rho}\in\tbook$, and $E$ is false, add the variables
    $\{e^{a\rho\ol{D}}_{i}\ | \i = 1,\hdots,\#\ol{C}\}$ to $\Lambda^\akb$, along with the inequalities:
    $\{0\leq e^{a \rho\ol{D}}_{i}\ | \i = 1,\hdots,\# a\}$ to $\Phi^\akb$. 
    Then, add to $\Phi^\akb$ the equalities:
    $$(\sum_{j=1}^{\#\ol{D}}r^{a \rho\ol{D}}_{ij}.x_j^E).e^{a\rho\ol{D}}_i+n^{aC}_i = 1.$$
  \item For all $a\in\names$, for all $\ol{D},\ol{E} \in\ol{\concepts}$, for all $i = 1,\hdots,\# a$, for all $\rho\in\roles$,
    $$\sum_{k=1}^{\#\ol{E}}\sum_{j=1}^{\#\ol{D}}r^{a\rho\ol{D}}_{ij}.e^{\ol{D}\ol{E}}_{jk} = 
    \sum_{j=1}^{\#\ol{D}}\sum_{k=1}^{\#\ol{E}}r^{a\rho \ol{E}}_{ik}.e^{\ol{E}\ol{D}}_{kj}.$$
\end{itemize}

This gives a set of equalities, $\Phi^\akb$, and inequalities $\Psi^\akb$ over the variables $\Lambda^\akb$,
such that any solution to these equations can be converted into an aleatoric belief model $\interp = (I, r, \ell)$ 
that satisfies $\akb$. This construction is as follows:
\begin{itemize}
  \item Enumerate $\ol{\concepts}$ as $X_1,\hdots, X_n$. 
    The set of individuals, $I$, is the union of $\{a_1,\hdots,a_{\# a}\ |\ a\in\names\}$ and the set 
    $$\{(x_1,\hdots, x_n)\in\omega^n |\ \forall i\leq n, x_i\leq \#X_i,\ \forall i,j\leq n, e^{X_iX_j}_{x_ix_j}>0\}$$  
  \item Given $\rho\in\roles$, and $a_i,b_j, (x_1,\hdots,x_n),(y_1,\hdots,y_n)\in I$:
    \begin{itemize}
      \item $\rho(a_i,b_j) = r^{a\rho b}_{ij}$
      \item $\rho(a_i,(x_1,\hdots,x_n)) = r^{a\rho X_1}_{i x_1}.e^{X_1X_n}_{x_1x_n}$
      \item $\rho((x_1,\hdots,x_n),(y_1,\hdots,y_n)) = e^{X_1X_n}_{x_1x_n}r^{X_nY_n}_{x_ny_n}e^{Y_nY_1}_{y_ny_1}$
      \item $\rho((x_1,\hdots,x_n),a_i) = 0$.
    \end{itemize}
    Note here, that from the constraints on $e^{X_iX_j}_{x_ix_j}$ the choice of indices in the above equalities is inconsequential. 
  \item Given $C\in\concepts$, and $a_i, (x_1,\hdots,x_n)\in I$:
    \begin{itemize}
      \item $C(a_i) = n^{aC}_i$, and
      \item $C((x_1,\hdots,x_n) = x^C_{x_i}$ where $C\in X_1$.
    \end{itemize}
\end{itemize}
It can be shown that $\interp$ satisfies the aleatoric knowledge base, 
by showing the semantics agree with the equations in $\Phi^\akb$.

Conversely, given an aleatoric belief model $\interp$ that satisfies the aleatoric knowledge base $\akb$, 
it can be shown that the set of equations is satisfiable by induction. 
Beginning with the leaf nodes $X_i$ of the DAG, 
ranges for $x^C_j$ can be found for each $C\in X_i$, that includes the values present in $\interp$.
Proceeding  back up the DAG,  the other variables can be solved by noting that $e^{X_iX_j}{nm}$ and $r^{X_iX_j}_{nm}$
provide the necessary degrees of freedom to guarantee a solution exists.

The number of variables, inequalities and equalities in the system $(\Lambda^\akb,\Psi^\akb,\Phi^\akb)$
is polynomial in the size of $\akb$, so determining if the system is satisfiable reduces to $\exists\mathbb{R}$ 
(the satisfiability of existentially quantified polynomial equations) which is in PSPACE. Theorem~\ref{thm:consistency} follows from this construction.

\begin{theorem}\label{thm:consistency}
  Given a simple aleatoric knowledge base $\akb = (\abook,\tbook)$ where $\tbook$ is acyclic, 
  it is possible to determine if $\akb$ is consistent with complexity PSPACE.
\end{theorem}

Finally, further reasoning tools can be inferred from \cite{icla}, 
where a sound and complete calculus is given for a propositional aleatoric logic, 
and another is proposed for a modal extension. 
These calculi may be applied as term rewriting systems to aid reasoning with aleatoric description logics.

\section{Expressivity}\label{sect:express}

\newcommand{\KDn}{\ensuremath{KD_n}}
\newcommand{\mac}{\ensuremath{MAC}}
This section presents a characterisation of aleatoric description logic as a probability space of
functional models introduced for $\ALC_1$. 
A model of $\ALC_1$ gives a very simple ontological description of the universe 
where at every individual every role is satisfied by precisely one individual.
The aleatoric belief models of \ADL\ describe a probability space of these simple descriptions, 
and the semantics of \ADL\ recursively define the likelihood of a formula holding in 
models sampled from this probability space.  

A probability space \cite{kolmogorov} is a tuple $(\Omega,\mathcal{F},\mathcal{P})$, 
where $\Omega$ is a set, $\mathcal{F}$ is a $\sigma$-algebra over $\Omega$ (the events),
and $\mathcal{P}$ is a probability measure on $\mathcal{F}$ that is countably additive.
The formal definition is below:
\begin{definition}\label{def:probability-space}
  A probability space is a tuple $(\Omega,\mathcal{F},\mathcal{P})$, where:
  \begin{itemize}
   \item $\Omega$ is a set.
    \item $\mathcal{F}$ is a $\sigma$-algebra over $\mathcal{F}$, so $\mathcal{F}\subset\wp(\Omega)$ such that:
      \begin{enumerate}
        \item $\Omega\in \mathcal{F}$
        \item If $A\in\mathcal{F}$ then $\Omega-A\in\mathcal{F}$
        \item If $A_0,A_1,\hdots\in\mathcal{F}$ then $\bigcup_{i=0}^\omega A_i\in\mathcal{F}$ 
          (i.e. $\mathcal{F}$ is countably additive).
      \end{enumerate}
    \item $\mathcal{P}:\mathcal{F}\longrightarrow[0,1]$ is a probability measure on $\mathcal{F}$ such that 
     $\mathcal{P}(\Omega) = 1$ and if $A_0,A_1,\hdots\in\mathcal{F}$ are pairwise disjoint, then 
          $\mathcal{P}(\bigcup_{i=0}^\omega A_i) = \sum_{i=0}^\omega\mathcal{P}(A_i)$.
  \end{itemize}
\end{definition}

The probability space corresponding to an aleatoric belief model is effectively 
the result of sampling interpretations of $\ALC_1$ (Definition~\ref{def:ALCSemantics})
from the model.
\begin{definition}\label{def:functional-space}
  Given a pointed aleatoric belief model $\interp_i = (I, r, \ell, i)$ 
  defined over a set of atomic concepts, $\concepts$, and roles, $\roles$, 
  we say an interpretation of $\ALC_1$, $\finterp = (J,c,f)$, is a {\em sampling} where:
  \begin{itemize}
    \item $J = I^+$ is a set of finite non-empty words over $I$, every word in $J$ begins with $i$.
    \item Given any $w\in I^*$, for all $wx\in J$, for all $\rho\in\roles$, the unique $w'\in J$ where $(wx,w')\in f(\rho)$ 
      is such that  $w'=wxy$ for some $y$ where $\rho(x,y)>0$
    \item Given any $w\in I^*$ for all $wx\in J$, for all $A\in\concepts$ $wx\in\kappa(C)$ implies $\ell(x,A)>0$.
  \end{itemize}
  We let $\Omega^{\interp_i}$ be the set of all samplings of $\interp_i$.
\end{definition}
In the context of Section~\ref{subsec:motivation}, each sampling in $\Omega^\interp$ 
corresponds to a scenario where the outcome of each die has been predetermined.

A $\sigma$-algebra over the probability space corresponds to a set of formulas, $\mathcal{L}$  of \ALC,
that is closed under countable unions, closed under complementation, closed under subformulas, and includes 
$\top$.
Every formula $\alpha$ describes the set of interpretations $\finterp_i$ in $\Omega^{\interp_i}$,
where $\finterp_i\models\alpha$. This set of interpretations is referred to as $\hat{\alpha}$.
Given a set of formulas $\mathcal{L}$ closed under countable unions, complementations and subformulas, 
let $\hat{\mathcal{L}} = \{\hat{\alpha}\ |\ \alpha\in\mathcal{L}\}$ be the algebra over $\Omega^{\interp_i}$
generated by $\mathcal{L}$.

The probability measure, $\mathcal{P}{\interp_i}$, over the algebra $\hat{\mathcal{L}}$, 
is a function mapping an element $\hat{\alpha}$ to the probability of sampling some interpretation, 
$\finterp_i$, from $\interp_i$, such that $\finterp_i\models\alpha$. 
The first part of this definition requires mapping all formulas of $\mathcal{L}$ to {\em acyclic alternating automata} \cite{loding2000alternating}.

\begin{definition}\label{def:AcycAltAut}
  An {\em acyclic alternating automaton} is given with respect to a finite set of atomic concepts $X'$ and a finite set of roles $R'$, 
  and is specified by a tuple $\mathcal{A} = (S_\exists, S_\forall, \delta_\exists,\delta_\forall, s_0)$ such that:
  \begin{itemize}
    \item $S_\exists$ is a finite set of {\em existential states}.
    \item $S_\forall$ is a finite set of {\em universal states}.
    \item $\delta_\exists:S_\exists\times\Sigma\hookrightarrow\wp(S_\forall)$ 
      is a partial function called the {\em existential transition function}.
    \item $\delta_\forall:S_\forall\times R\hookrightarrow S_\exists$ 
      is a partial function called the {\em universal transition function}.
    \item $s_0\in S_\exists$ is the initial state. 
  \end{itemize}
  Additionally, the automaton must be {\em acyclic} so there is no sequence 
  $e_0, \ell_0, u_0,\rho_0, e_1, \ell_1, \hdots u_n,\rho_n$ such that 
  $u_i\in\delta_\exists(e_i,\ell_i)$, $e_{i+1}\in\delta_\forall(u_i,\rho_i)$ and $e_0=\delta_\forall(u_n,\rho_n)$.
  The automaton $\mathcal{A}$ acts on pointed interpretations of $\ALC_1$, $\mathcal{I}_i = (I,c,r,i)$ 
  in the form of a game $\mathcal{G}(\mathcal{A},\mathcal{I}_i)$.
  The game is a series of {\em positions} where every position is a pair in $(S_\exists\cup S_\forall)\times I$, 
  and two players, $\exists$ and $\forall$, take turns choosing the next position.
  The game starts at an initial position $(s_0,i)$. 
  In a position $(s,i)$ where $s\in S_\exists$, player $\exists$ chooses $t\in \delta_\exists(s,\{A\in X' |\ i\in c(A)\})$ 
  (if it exists), and the next position is $(t,i)$.
  In a position $(s,i)$ where $s\in S_\forall$, player $\forall$ chooses some $\rho\in R'$ and 
  if $\delta_\forall(s,\rho)$ exists, the next position becomes 
  $(\delta_\forall(s,\rho), j)$, where $(i,j)\in r(\rho)$.
  If at any position a the player is unable to make a move because the transition function is not defined, 
  then that player loses the game.
  As the automaton is acyclic, every game is determined, 
  so given any automaton $\mathcal{A}$ and any pointed $\ALC_1$ interpretation $\mathcal{I}$, 
  either $\exists$ or $\forall$ will have a winning strategy.
  If and only if $\exists$ has a winning strategy in the game $\mathcal{G}(\mathcal{A},\mathcal{I}_i)$, 
  we say $\mathcal{A}$ {\em accepts} $\mathcal{I}_i$.

\end{definition}

The alternating automata give a computational representation of formulas of $\ALC_1$, 
which are used to characterise a probability mass function over the $\sigma$-algebra, $\hat{\mathcal{L}}$.

\begin{lemma}\label{lem:alt-aut}
  For every formula $\alpha$ of $\ALC_1$, there is an acyclic alternating automaton
  that accepts exactly the pointed functional interpretations $\finterp_i$, where $\finterp_i\models\alpha$. 
\end{lemma}
\begin{proof}
  A formula, $\alpha$, is first converted to {\em positive normal form}, $\alpha^+$ where negations are only applied to atoms, 
  and $\sqcup$ disjunctions and $\bot$ are added to the syntax. 
  The positive normal form formulas are then mapped to an acyclic alternating automaton, $\mathcal{A}^\alpha$ 
  such that the alternating automaton accept precisely the functional interpretations of $\ALC_1$ that satisfy $\alpha$.
  

  Given any formula $\alpha$ in positive normal form it is possible to construct an acyclic alternating automaton, $\mathcal{A}^\alpha$, 
  such that player $\exists$ has a winning strategy in the game $\mathcal{G}(\mathcal{A}^\alpha,\mathcal{I}_i)$
  precisely when $\mathcal{I}_i\models\alpha$. Let $X'$ be the set of atomic concepts appearing in $\alpha$, and
  $R'$ be the set of roles that appear in $\alpha$. A recursive construction is given below, 
  with justifications for the construction assuming a model $\mathcal{I}_i = (I,c,r,i)$. 
  The non-trivial constructions are below: 
  \smallskip
    $\beta_1\sqcap\beta_2$: 
      Suppose $\mathcal{A}^{\beta_i} = (S^i_\exists,S^i_\forall,\delta^i_\forall,\delta^i_\exists, s_0^i)$, for $i = 1,2$.
      Then $\mathcal{A}^{\beta_1\sqcap\beta_2} = 
      (S_\exists, S_\forall,\delta_\exists,\delta_\forall, (s_0^1,s_0^2))$,
      where 
      \begin{enumerate}
        \item $S_\exists =S^1_\exists\cup S^2_\exists\cup\ S^1_\exists\times S^2_\exists$
        \item $S_\forall = S^1_\forall\times S^2_\forall\cup S^1_\forall\cup S^2_\forall$
        \item for $i = 1,2,$ for all $s\in S^i_\exists$, for all $\ell\in\wp(X')$, $\delta_\exists(s,\ell) = \delta^i_\exists(s,\ell)$, 
          for all $(s,s')\in S^1_\exists\times S^2_\exists$, 
          $\delta_\exists((s,s'),\ell) = \delta^1_\exists(s,\ell)\times\delta^2_\exists(s',\ell)$,
          and $\delta_\exists(\beta_1\sqcap\beta_2,\ell) = \delta_\exists^1(s_0^1,\ell)\times \delta_\exists^2(s_0^2,\ell$ 
        \item for $i = 1,2$, for all $s\in S^i_\forall$, for all $\rho\in R'$, $\delta_\forall(s,\rho) = \delta^i_\forall(s,\rho)$,
          for all $(s,s')\in S^1\times S^2_\forall$, for all $\rho\in R'$,
          \begin{enumerate}
            \item if $\delta^1_\forall(s,\rho)\neq\emptyset$ and $\delta^2_\forall(s',\rho)\neq\emptyset$ then
              $\delta_\forall((s,s'),\rho) = (\delta^1_\forall(s,\rho),\delta^2_\forall(s',\rho))$, and
            \item if $\delta^1_\forall(s,\rho)\neq\emptyset$ and $\delta^2_\forall(s',\rho)=\emptyset$ then
              $\delta_\forall((s,s'),\rho) = \delta^1_\forall(s,\rho)$, and
            \item if $\delta^1_\forall(s,\rho) =\emptyset$ and $\delta^2_\forall(s',\rho)\neq\emptyset$ then
              $\delta_\forall((s,s'),\rho) = \delta^2_\forall(s',\rho)$.
          \end{enumerate}
      \end{enumerate}
      The construction of the automaton is such that players $\exists$ and $\forall$ simultaneously play 
      $\mathcal{G}(\mathcal{A}^{\beta_1},\mathcal{I}_i)$ and $\mathcal{G}(\mathcal{A^{\beta_2}},\mathcal{I}_i)$.
      Player $\exists$ must always play a move in both games, but player $\forall$ 
      only needs to play a move in one game for each role $\rho\in R'$. 
      If either $\mathcal{I}_i\not\models\beta_1$ or $\mathcal{I}_i\not\models\beta_2$, 
      player $\forall$ has a winning strategy in one game, and following that strategy will lead to a win in
      $\mathcal{G}(\mathcal{A}^{\beta_1\sqcap\beta_2},\mathcal{I}_i)$. 
      Conversely, if both $\mathcal{I}_i\models\beta_1$ and $\mathcal{I}_i\models\beta_2$, 
      then player exists has a winning strategy in both games, and following these strategies simultaneously
      will lead to a scenario where player $\forall$ cannot make a move.
    
    \smallskip
    $\alpha\sqcup\beta$:
      Suppose $\mathcal{A}^{\beta_i} = (S^i_\exists,S^i_\forall,\delta^i_\forall,\delta^i_\exists, s_0^i)$, for $i = 1,2$.
      Then $\mathcal{A}^{\beta_1\sqcup\beta_2} = 
      (S_\exists, S_\forall,\delta_\exists,\delta_\forall, \beta_1\sqcup\beta_2)$,
      where 
      \begin{enumerate}
        \item $S_\exists =S^1_\exists\cup S^2_\exists\cup\{\beta_1\sqcup\beta_2\}$
        \item $S_\forall = S^1_\forall\cup S^2_\forall$
        \item for all $s\in S^1_\exists$, for all $\ell\in\wp(X')$, $\delta_\exists(s,\ell) = \delta^1_\exists(s,\ell)$, 
          for all $s\in S^2_\exists$, $\delta_\exists(s,\ell) = \delta^2_\exists(s,\ell)$,
          and $\delta_\exists(\beta_1\sqcup\beta_2,\ell) = \delta_\exists^1(s_0^1,\ell)\cup\delta_\exists^2(s_0^2,\ell)$.
        \item for all $s\in S^1_\forall$, for all $\rho\in R'$,  $\delta_\forall(s,\rho) = \delta^1_\forall(s,\rho)$, 
          and for all $s\in S^2_\forall$, $\delta_\forall(s,\rho) = \delta^2_\forall(s,\rho)$.
      \end{enumerate}
      In the game $\mathcal{G}(\mathcal{A}^{\beta_1\sqcup\beta_2},\mathcal{I}_i)$, 
      player $\exists$ is effectively able to choose whether to play the game 
      $\mathcal{G}(\mathcal{A}^{\beta_1},\mathcal{I}_i)$ or
      $\mathcal{G}(\mathcal{A}^{\beta_2},\mathcal{I}_i)$. 
      As player $\exists$ has a winning strategy in at least one of these games if and only if 
      $\mathcal{I}_i\models\beta_1\sqcup\beta_2$, the correspondence holds.
    
    \smallskip
    $\exists \rho.\beta$:
      Suppose $\mathcal{A}^{\beta} = (S^\beta_\exists,S^\beta_\forall,\delta^\beta_\forall,\delta^\beta_\exists, s_0\beta)$.
      Then \newline $\mathcal{A}^{\exists\rho.\beta} = (\wp(X'), S_\exists, S_\forall,\delta_\exists,\delta_\forall, \exists\rho.\beta)$, where:
      \begin{enumerate}
        \item $S_\exists = S^\beta_\exists\cup\{\exists\rho_\beta\}$
        \item $S_\forall = S^\beta_\forall\cup\{\forall\rho_\beta\}$
        \item for all $s\in S_\exists$, for all $\ell\in\wp(X')$, $\delta_\exists(s,\ell) = \delta^\beta_\exists(s,\ell)$, and $\delta_\exists(\exists\rho_\beta,\ell) = \{\forall\rho\beta\}$.
        \item for all $s\in S_\forall, $ for all $\rho'\in R'$, $\delta_\forall(s,\rho') = \delta^\beta_\forall(s,\rho')$ if $s\in S^\beta_\forall$, 
          and $\delta_\forall(\forall\rho.\beta,\rho')$ is not defined if $\rho\neq\rho'$ and 
          $\delta_\forall(\forall\rho\beta,\rho) = s_0^\beta$
      \end{enumerate}
      In this construction, the first move for both $\exists$ and $\forall$ are effectively pre-determined.
      Player $\exists$ chooses $\forall\rho.\beta$, and then if player $\forall$ chooses any role other than $\rho$, 
      they immediately lose.
      Therefore player $\forall$ chooses $\rho$, and as there is a single $j\in I$ where $(i,j)\in r(\rho)$ 
      the next position of the game is necessarily $(s_0^\beta,j)$. 
      From there, the game is equivalent to the game $\mathcal{G}(\mathcal{A}^\beta,\mathcal{I}_j)$,
      which, by induction, has a winning strategy for $\exists$ if and only if $\mathcal{I}_j\models\beta$, 
      which in turn is equivalent to $\mathcal{I}_i\models\exists\rho\beta$, as required. 
  
  \bigskip
  As discussed above the constructions match the semantics of the associated formulas. 
  We can also see that the automata are acyclic, 
  as the label of any successor state is a subformula of the label of the current state.
\end{proof}

To define a probability mass function, it is required that there is no redundancy in the automaton,
so we will assume that the automaton is bisimulation minimal \cite{loding2000alternating}.

Given $\mathcal{A} = (S_\exists,S_\forall,\delta_\exists,\delta_\forall,s_0)$ is a minimal acyclic alternating automaton 
defined with respect to the set of atomic concepts $X'$ and the roles $R'$, 
we characterise the set of accepted models by a set of finite $X'$-$R'$-trees:
\begin{definition}\label{def:tree}
  Given the finite set of atomic concepts $X'$ and the finite set of roles $R'$,
  and an acyclic alternating automaton 
  $\mathcal{A} = (S_\exists,S_\forall,\delta_\exists,\delta_\forall,s_0)$, 
  a finite set of words $T\subset\wp(X')(R'\wp(X'))^*$ is an {\em accepted tree of} $\mathcal{A}$ if
  there is some function $\lambda:T\rightarrow S_\exists$ where:
  \begin{itemize}
    \item there is some $Y\subseteq X'$ for all $t\in T$, $t = Yw$, and $\lambda(Y) = s_0$.
    \item for all $\rho\in R'$, for all $Y\subseteq X'$, if $w\rho Y\in T$, then $w\in T$.
    \item for all $w\in T$, for all $\rho\in R'$, for all $Y, Z\subseteq X'$ if $w\rho Y, w\rho Z\in T$, then $Y=Z$
    \item for all $Y,Z\subseteq \wp(X')$, for all $wY\in T$, for all $\rho\in R'$, if $\lambda(wY) = s$, and $\lambda(wY\rho Z)\in T$ 
      then there is some $s'\in \delta_\exists(s, Y)$ such that $\lambda(wY\rho Z) = \delta_\forall(s',\rho)$.
    \item for all $Y\subseteq \wp(X')$, for all $wY\in T$, for all $\rho\in R'$, if $\lambda(wY) = s$, 
      and $\forall Z\subseteq X'$, $\lambda(wY\rho Z)\notin T$ 
      then there is some $s'\in \delta_\exists(s, Y)$ such that $\delta_\forall(s',\rho)$ is undefined.
  \end{itemize}
  The {\em language recognised by} $\mathcal{A}$ is the set $\mathcal{T}$ of accepted trees of $\mathcal{A}$, $T$,
  where no proper subtree of $T$ is accepted by $\mathcal{A}$.
\end{definition}
The language recognised by $\mathcal{A}$ is effectively the prefix of all interpretations that are accepted by $\mathcal{A}$,
and gives a finite representation of those interpretations.

A probability measure $\mathcal{P}^{\interp_i}$ for the probability space is defined as below.
\begin{definition}\label{def:probMeasure}
  Let $\interp_i = (I, r, \ell, i)$ be a pointed aleatoric belief model, and let $\mathcal{L}$ be some set of \ALC\ formulas 
  defined over a finite set of atomic concepts $X'$ and a finite set of roles, $R'$, and suppose $\mathcal{L}$ is closed under
  countable unions, complementations and subformulas, where $\top\in\mathcal{L}$. 
  The {\em probability measure} $\mathcal{P}^{\interp_i}:\hat{\mathcal{L}}\longrightarrow[0,1]$ is defined such that:
  $\mathcal{P}^{\interp_i}(\hat{\alpha}) = \sum_{T\in\mathcal{T}}P^{\interp_i}(T)$, 
  where $\mathcal{T}$ is the language recognised by $\mathcal{A}^\alpha$ 
  and $P^{\interp_i}:\mathcal{T}\rightarrow[0,1]$ is defined recursively as:
  \begin{eqnarray*}
    P^{\interp_i}(Y) &=& (\prod_{y\in Y}\ell(i,y)).(\prod_{y\in X'-Y}(1-\ell(i,y)))\\
    &&{\rm where}\ Y\in \wp(X')\\
    P^{\interp_i}(T) &=& P^{\interp_i}(Y).\sum_{\rho\in R'',\ j\in I}\rho(i,j).P^{\interp_j}(T^\rho)\\
      &&{\rm where}\ T^\rho = \{Y\rho T^\rho\ |\ \rho\in R''\subseteq R'\}
  \end{eqnarray*}
  noting that $\sum_{x\in\emptyset} x =0$ and $\prod_{x\in\emptyset}x = 1$.
\end{definition}

The probability measure of the algebra element $\hat{\alpha}$, 
is the probability of sampling a functional interpretation from $\interp_i$
that the automaton $\mathcal{A}^\alpha$ will accept.

\begin{lemma}\label{lem:ProbSpace}
  Let $\interp_i$ be a pointed aleatoric belief model, 
  and $\mathcal{L}$ be a set of \ALC\ formulas closed under countable disjunctions, complementations and subformulas.
  The tuple $(\Omega^{\interp_i},\hat{\mathcal{L}},\mathcal{P}^{\interp_i})$ is a probability space.
\end{lemma}
\begin{proof}
  From the semantics of $\ALC_1$ it follows that $\hat{\mathcal{L}}$ is closed under complements and countable unions,
  and contains the entire space $\Omega^{\interp_i} = \hat{\top}$.
  It remains to show that $\mathcal{P}^{\interp_i}$ is a probability measure. 
  From Definition~\ref{def:interpretation} it follows that $\ell(i,y)\in[0,1]$ and $r(\rho,i)$ 
  is a probability distribution over $I$, so $\mathcal{P}^{\interp_i}$ will be in the valid range, $[0,1]$.
  As all formulas in $\mathcal{L}$ are finite, for the countable additivity requirement 
  it suffices to show that measure of any union of two disjoint sets is equal to the sum of the measures of the sets.
  Again this follows directly from the definition.
  If $\hat{\alpha}$ and $\hat{\beta}$ are disjoint then the languages recognised by $\mathcal{A}^\alpha$ and 
  $\mathcal{A}^\beta$ will be disjoint since a common tree in the languages would be able to generate an interpretation 
  over which both automata had accepting runs. 
  Definition~\ref{def:probMeasure} describes the likelihood of an interpretation being sampled from $\interp_i$ 
  that matches the recognised language of an automaton, 
  so the probability of sampling from either disjoint set is the sum of the probabilities of sampling from each set.
\end{proof}

Therefore, $\interp_i$ describes a probability space, 
$(\Omega^{\interp_i},\mathcal{L},\mathcal{P}^{\interp_i})$, 
of $\ALC_1$ models for any set of formulas $\mathcal{L}$ closed under countable disjunctions, 
complementations and subformulas. 
As such, $\mathcal{P}^{\interp_i}$ can be seen as 
a probability distribution of ontological descriptions of a universe
that is consistent with the semantic interpretation of the formulas in $\mathcal{L}$.

The final part of this section shows that \ADL\ 
characterises the likelihood of formulas of $\ALC_1$ in this representation.
It is shown that for any $\ALC$ formula $\alpha$, 
there is a corresponding \ADL\ formula $\alpha^*$ such that
$\mathcal{P}^{\interp_i}(\hat{\alpha}) = \interp_i(\alpha^*)$.

\begin{definition}\label{def:alt2adl} 
  Let $\mathcal{A} = (S_\exists,S_\forall,\delta_\exists,\delta_\forall,s_0)$
  be an acyclic alternating automaton
  defined over atomic concepts $X' = \{x_0,\hdots, x_n\}$ and roles 
  $R' = \{\rho_0,\hdots,\rho_m\}$. For each state $s\in S_\exists$,
  $\tau(s)$ is an $\ADL$ formula defined as follows:
  
  $\tau(s) = \tau^\emptyset_0$ where for $Y\subseteq X'$ and $i<n$
      $\tau^Y_i = \ite{x_i}{\tau^{Y\cup\{x_i\}}_{i+1}(s)}{\tau^Y_{i=1}}$
      and $\tau^Y_n = \ite{x_n}{\Lambda(\delta_\exists(s,Y\cup\{x\}))}{\Lambda(\delta_\exists(s, Y))}$.
  
  Given $S = \{t_0,\hdots,t_k\}\subset S_\forall$, $\Lambda(S) = \Lambda_{(0,0)}(\mu_0)$
      where for $(i,j)$ such that $i\leq m$ and $j\leq k$, and for $\mu:\{0,\hdots,m\}\longrightarrow\lang{\ADL}$:
      \begin{eqnarray*}
        \Lambda_{(i,j)}(\mu) &=& \ite{\alpha}{\beta}{\gamma}\quad{\rm where}\\ 
        \alpha &=& \cond{\tau(\delta_\forall(t_j,\rho_i))}{\mu(\rho_i)}{\rho}\\
        \beta &=& \Lambda_{(i+1,j)}(\mu^{[i\mapsto \mu(i)\sqcap\tau(\delta_\forall(t_j,\rho_i))]})\\
        \gamma &=& \Lambda_{(0,j+1)}(\mu^{[i\mapsto \mu(i)\sqcap\lnot\tau(\delta_\forall(t_j,\rho_i))]})
      \end{eqnarray*}
  Here, $\mu_0(i) = \top$ for $i\leq m$, 
  $\mu^{[i\mapsto \alpha]}(j) = \mu(j)$ for all $j\neq i$, 
  and $\mu^{[i\mapsto \alpha]}(i) = \alpha$. Also, $\Lambda_{(n+1,j)}(\mu) = \top$ and $\Lambda_{(i,m+1)} = \bot$.

  Let $\tau(\mathcal{A}) = \tau(s_0)$ and $\alpha^* = \tau(\mathcal{A}^\alpha)$.
\end{definition}

Definition~\ref{def:alt2adl} motivates the syntactic structure of $\ADL$.
For every state $s\in S_\exists$, $\tau_s$ uses the if-then-else operator to build a tree that evaluates
each atomic concept exactly once, so the leaves of this tree determine an element $Y\subseteq X'$ that was sampled.
Given $Y$ and $s$, the set $S=\delta_\exists(s,Y)$ of possible $S_\forall$ successors is computed, 
and then for each possible successor $t\in S$, for every $\rho\in R'$, 
$\Lambda(S)$ evaluates the likelihood of $\tau(\delta_\forall(t,\rho))$ holding at the sampled $\rho$-successor of the current state.
Just as the if-then-else operator was used to ensure that each concept was sampled only once, 
here the marginalisation operator is used to require that each $\rho$-successor is effectively sampled only once.
This is done by using the function $\mu$ to retain what formula was evaluated previously at the sampled $\rho$-successor, 
and using the marginalisation operator to condition on that formula being true. 
While there is no guarantee that the sampled $\rho$-successor will be exactly the same, 
the sample will be taken from the same distribution so the average effect will be the same. 
The if-then-else operator and the marginalisation operators are essential for capturing these representations.

\begin{lemma}\label{lem:alc2adl}
  Let $\alpha$ be a formula of \ALC\, $\interp_i$ be an aleatoric belief model, 
  and $\mathcal{L}$ be a set of \ALC\ formulas closed under countable disjunctions, 
  complementations and subformulas, where $\alpha\in\mathcal{L}$.
  In the probability space $(\Omega^{\interp_i},\hat{\mathcal{L}},\mathcal{P}^{\interp_i})$  
  $$\mathcal{P}^{\interp_i}(\hat{\alpha}) = \interp_i(\alpha^*).$$
\end{lemma}
\begin{proof}(Sketch)
  The construction of $\alpha^*$ is such that every relation and every concept in $\alpha$ is sampled
  once per individual mirroring the probability measure $\mathcal{P}^{\interp_i}$ (Definition~\ref{def:probMeasure}).
  This can be seen in the definition of $\tau(S)$ where $s\in S_\exists$: 
  each atomic concept is sampled, and the formula branches on the result, 
  effectively creating a tree where the leaves describe the set of atomic concepts sampled to be true.
  For example, if $X' = \{x,y\}$, then 
  $$\tau(s) = \ite{x}{\ite{y}{\Lambda^{\{xy\}}_s}{\Lambda^{\{x\}}_s}}{\ite{\{y\}}{\Lambda^{01}_s}{\Lambda^{\emptyset}_s}}$$
  where $\Lambda^Y_s = \Lambda(\delta_\exists(s,Y)$. 
  Expanding the semantics (Definition~\ref{def:semantics}):
  $$\interp_i(\tau_s) = \sum\left\{\begin{array}{l}
    \ell(i,x).\ell(i,y).\interp_i(\Lambda^{\{xy\}}_s)\\
    \ell(i,x).(1-\ell(i,y)).\interp_i(\Lambda^{\{x\}}_s)\\
    (1-\ell(i,x)).\ell(i,y).\interp_i(\Lambda^{\{y\}}_s)\\
    (1-\ell(i,x)).(1-\ell(i,y)).\interp_i(\Lambda^{\emptyset}_s)
  \end{array}\right\}$$
  which agrees with the definition of $P^{\interp_i}(Y)$ (Definition~\ref{def:probMeasure}).

  The definition of $\Lambda(S)$ recursively defines the probability of sampling a set of $R'$-successors
  that will ultimately be accepted by the acyclic alternating automaton. 
  It does this by checking the successors in turn. 
  In the definition of $\Lambda_{(0,0)}(\mu_0)$, the $\rho_0$-successor is sampled 
  to test if the first $\delta_\forall$ state ($t_0$) will lead to an accepting run on that branch:
  $$\alpha = \cond{\tau(\delta_\forall(t_0,\rho_0))}{\top}{\rho_0}.$$
  When $\alpha$ holds, $\mu$ is updated to record that $\tau(\delta_\forall(t_0,\rho_0))$ was true at the $\rho_0$ successor, 
  and the formula $\Lambda_{(0,0)}^{\mu'}$ then samples a $\rho_1$ successor 
  to test if $t$ also leads to an accepting run on the $\rho_1$ branch. 
  When $\alpha$ does not hold, the state $t_0$ will not lead to an accepting run, 
  so all other successors are ignored and attention moves to $t_1$ and the process repeats.
  However, this time when a $\rho_0$ successor is sampled, 
  it should be the same successor that was tested when investigating $t_0$. 
  All things being equal, any $\rho_0$-successor where 
  $\tau(\delta_\forall(t_0,\rho_))$ does not hold is sufficient,
  so in this case the $\rho_1$-successor is marginalised by the possibility:
  $$\gamma = \cond{\tau(\delta_\forall(t_1,\rho_0)}{\lnot\tau(\delta_\forall(t_0,\rho_0))}{\rho_0}.$$
  Expanding the fragment of semantics (Definition~\ref{def:semantics}), assume that $R' = \{\rho_0\}$ 
  and the set of $S_\forall$ successors are $\{t_0,t_1\}$ gives the following calculations:
  \begin{eqnarray*}
      &&\interp_i(\Lambda_{(0,0)}(\mu_0))\\
      &=& \interp_i(\ite{\cond{\alpha_0}{\top}{\rho_0}}{\top}{\cond{\alpha_1}{\lnot\alpha_1}{\rho_0}})\\
      &=& \interp_i(\expect_{\rho_0}\alpha_0)+
         \interp_i(\lnot\expect_{\rho_0}\alpha_0).
           \frac{\sum_{j\in I}\rho_0(i,j).\interp_i(\alpha_1).\interp_i(\lnot\alpha_0)}{\sum_{j\in I}\rho_0(i,j).\interp_i(\lnot\alpha_0)}\\
      &=& \interp_i(\expect_{\rho_0}\alpha_0)+\interp_i(\expect_{\rho_0}(\alpha_1\sqcap\lnot\alpha_0)),
  \end{eqnarray*}
  where $\alpha_0 = \tau(\delta_\forall(t_0,\rho_0))$ and $\alpha_1 = \tau(\delta_\forall(t_1,\rho_0))$.
  Therefore, each probability of an accepting run is summed to give the probability of sampling a functional model recognised by the automaton, 
  as in Definition~\ref{def:probMeasure}.  
\end{proof}

The following theorem is a direct consequence of 
Lemmas~\ref{lem:alt-aut},~\ref{lem:ProbSpace}~and~\ref{lem:alc2adl}
\begin{theorem}\label{thm:adl2probspace}
 Given a formula $\alpha$ of \ALC\, and an aleatoric belief model $\interp_i$,
  there exists: a formula $\alpha^*$ that is logically equivalent to $\alpha$ in \ALC;
  and probability space $(\Omega^{\interp_i},\mathcal{F},\mathcal{P}^{\interp_i})$ where:
  $\Omega^{\interp_i}$ is a set of functional \ALC\ models;
  $\mathcal{F}$ is an algebra over $\Omega^{\interp_i}$ consisting of an element $\hat{\beta}$ for every \ALC\ formula $\beta$; and
  $\mathcal{P}^{\interp_i}$ is a probability measure on $\mathcal{F}$ derived from $\interp_i$.
  This probability space is such that $\mathcal{P}^{\interp_i}(\hat{\alpha}) = \interp_i(\alpha^*)$.
\end{theorem}

Note, this correspondence only goes in one direction since the translation from $\ALC$ to $\ADL$ is not surjective. 
A reverse translation is not possible as $\ADL$ can represent properties not expressible by $\ALC$: 
for example $C\sqcap C \neq C$ in $\ADL$. 
However, it is possible to give a correctness preserving translation from $\ALC$ to $\ADL$, 
and \cite{icla} presents a correctness proof for the modal case.

Finally a direct encoding of \ALC\ may be given in \ADL\ by projecting all concept probabilities to 0 or 1, 
and using a uniform distribution to simulate quantification over roles.
This construction is sketched below (see \cite{icla} for a correctness proof in the modal case).

Let $\mathcal{I}_i = (I, c, r', i)$ be an interpretation of $\ALC$, where $I$ is finite,
and define $\interp^{\mathcal{I}}_i = (I, r, \ell, i)$ to be the pointed aleatoric belief model,
where 
\begin{itemize}
  \item $\rho(i,j) = \frac{1}{|\{k\in I\ |\ (i,k)\in r'(\rho))\}|}$ if $(i,j)\in r'(\rho)$ and $0$ otherwise.
  \item $\ell(i,A) = 1$ if $i\in c(A)$ and $0$ otherwise.
\end{itemize}

Once atom distributions have been mapped to Booleans and role distributions have been made uniform,
it is straightforward to show that the abbreviations in Table~\ref{tab:abbreviations} preserve their meaning
so $\mathcal{I}_i\models \alpha$ if and only if $\interp^{\mathcal{I}}_i(\alpha) = 1$.

\section{Learning}\label{sect:learning}

An aleatoric belief model describes an agent's beliefs and prior assumptions 
and the agent may update these beliefs based on observations, via Bayesian conditioning.
This section will introduce two learning mechanisms, {\em role learning} and {\em concept learning} whereby 
an agent may update the distribution of individuals fulfilling a role, 
and also update the aleatoric probabilities associated with an atomic concept at an individual. 
These mechanisms are unique to $\ADL$ and provides a compelling advantage over alternative
probabilistic description logics \cite{ceylan,lukasiewicz,lutz-DL,riguzzi2015probabilistic,pozzato2019typicalities}. 

In this sense, given a consistent aleatoric belief set, 
an aleatoric belief model acts can be chosen as a Bayesian prior (for example, by assigning a uniform prior to undeclared probabilities). 
Then observations of individuals, concepts, and roles 
can be used to compute posterior likelihoods, which update the aleatoric belief model.
Therefore, even a quite basic belief model can be refined over time to detect and learn 
subtle relationships between concepts and individuals.

The two learning mechanisms discussed here are {\em role learning} and {\em concept learning}.
In this setting, it is assumed that the agent has prior beliefs, 
and these beliefs involve some uncertainty (modelled aleatorically). 
The agent also has an observation, which is a formula of $\ADL$.
While observation can be considered a fact, and not subject to uncertainty,
it is also able to inform the agent about there own beliefs, 
via Bayesian updating.

Consider the example from Subsection~\ref{sbsect:example}, 
with Hector, Igor, and Julia, 
each of whom maybe infected with a virus or not. 
Suppose that this model is a representation of Hector's beliefs,
and also, that Hector is informed via a contact-tracing exercise 
that they have come into contact with an infected person.
This information alone, allows Hector to refine the belief model, through {\em role learning} 

\subsection{Role learning}

Role learning refines the probability distribution associated with a role $\rho$.
For a pointed aleatoric belief model, $\interp_i = (I,r, \ell,i)$, for every $j\in I$,
$\rho(i,j)$ is the prior probability that $j$ fulfils the role of $\rho$ for $i$.
Given an observation which is an $\ADL$ formula of the form $\cond{\alpha}{\top}{\rho}$,
$\interp_j(\alpha)$ is the probability of this observation holding, 
given $j$ fulfils the role of $\rho$ for $i$. 
Via Bayes' rule, it follows that the probability of $j$ fulfilling the role of $\rho$ for $i$, 
{\em given the observation} is:
$$\rho'(i,j) = \frac{\rho(i,j) \cdot \interp_j(\alpha)}{\interp_i(\cond{\alpha}{\top}{\rho})}$$
(the prior probability of $j$ is multiplied by the probability of $\alpha$ given $j$, 
divided by the probability of $\alpha$).

\begin{definition}\label{def:role-learning}
  Let $\interp_i = (I,r,\ell,i)$ be an aleatoric belief model, and $\phi = \cond{\alpha}{\beta}{\rho}$
  an observation, made at $i$. 
  The $\phi${\em -update of} $\interp_i$ is the aleatoric belief model 
  $\interp^\phi_i = (I,r^{i,\alpha},\ell,i)$, where for all $\rho'\neq\rho$ and $j\neq i$,
  $r^{i,\phi}(\rho',j) = r(\rho,j)$ and for all $j\in I$ 
  $$r^{i,\phi}(\rho,i)(j) = \frac{\rho(i,j) \cdot \interp_j(\alpha)}{\interp_i(\cond{\alpha}{\beta}{\rho})}.$$
\end{definition}

Thus an agent with an {\em aleatoric} model of the world may update their {\em epistemic} uncertainty
of the distribution of roles, via Bayesian conditioning. 
The $\phi$-update of $\interp_i$ is the agent's posterior model of the world.

Given the example in Subsection~\ref{sbsect:example}, 
suppose that Hector's belief model is $\interp = (I,r,\ell,i)$, and Hector is 
informed that the contact has tested positive for the virus.
Hector is also informed that the test used has a 10\% false positive rate, 
so Hector's belief model now includes an atomic concept $\mathit{FP}$ that is 0.1 everywhere.
Let $\phi = \cond{\ite{FP}{\top}{V}}{\top}{c}$ and then the $\phi$-update of $\interp_{\hO}$
is computed by:
$$r^{\hO,\phi}(c,\mathtt{H_0})(j) = \frac{c(i,j) \cdot (0.1+0.9 \cdot \interp_j(V)}{\interp_{\hO}(\cond{\ite{FP}{\top}{V}}{\top}{c})}.$$
Substituting in the values from Table~\ref{tab:ex-init-prob}, 
Hector is able to discount the possible individuals without a virus and condition the distribution 
for $contact$ accordingly. The $\phi$-update of $\interp_{\hO}$ is represented in Figure~\ref{fig:phi-update}.

\begin{figure}
  \begin{center}
    \scalebox{0.8}{
    \begin{tikzpicture}
      \draw (2,5) node[circle,draw](h0) {\tiny{ $\begin{array}{c}\hO\\{\mathit V}:0.0\\{\mathit F}:0.1\end{array}$ }};
      \draw (2,2) node[circle,draw](h1) {\tiny{ $\begin{array}{c}\hI\\{\mathit V}:1.0\\{\mathit F}:0.6\end{array}$ }};
      \draw (5,6) node[circle,draw](i0) {\tiny{ $\begin{array}{c}\iO\\{\mathit V}:0.0\\{\mathit F}:0.3\end{array}$ }};
      \draw (7,4) node[circle,draw](i1) {\tiny{ $\begin{array}{c}\iI\\{\mathit V}:1.0\\{\mathit F}:0.8\end{array}$ }};
      \draw (8,2) node[circle,draw](j0) {\tiny{ $\begin{array}{c}\jO\\{\mathit V}:0.0\\{\mathit F}:0.2\end{array}$ }};
      \draw (6,0) node[circle,draw](j1) {\tiny{ $\begin{array}{c}\jI\\{\mathit V}:1.0\\{\mathit F}:0.9\end{array}$ }};
      \draw[dashed] (h0) -- (h1) node[midway, right](pt-h) {$\mathtt{id}$} node[near start, left] {0.1} node[near end,left] {0.9};
      \draw[dashed] (i0) -- (i1) node[midway, below left](pt-i) {$\mathtt{id}$} node[pos=0.1, right] {$\mathbf{0.1}$} node[pos=0.6, right] {$\mathbf{0.9}$};
      \draw[dashed] (j0) -- (j1) node[midway, above left](pt-j) {$\mathtt{id}$} node[pos=0.5, right] {$\mathbf{0.05}$} node[pos=0.95, right] {$\mathbf{0.95}$};

      \draw[thick,<->] (pt-h) -- (pt-i) node[midway, below] {$\mathtt{c}$} node[near start, above] {0.4} node[near end,above] {$\mathbf{0.25}$};
      \draw[thick,<->] (pt-i) -- (pt-j) node[midway, right] {$\mathtt{c}$} node[near start, left] {0.6} node[near end,left] {0.6};
      \draw[thick,<->] (pt-j) -- (pt-h) node[midway, above] {$\mathtt{c}$} node[near start, below] {$\mathbf{0.75}$} node[near end,below] {0.4};
    \end{tikzpicture}
    }
  \end{center}
  \caption{The $\phi$-update of the aleatoric belief model in Figure~\ref{fig:example}, after Hector is told a contact has tested positive for the virus.
  The updated values are bold.}\label{fig:phi-update}
\end{figure}
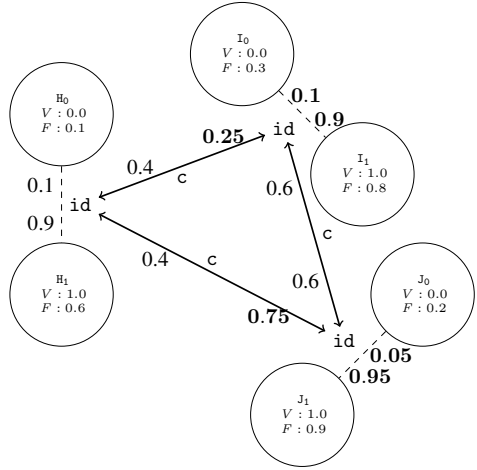

It may seem odd that we ``observe'' a marginalised formula, as there is non concept of conditional probability in a single observation.
However this makes sense when it is interpreted as an instance of selection bias. 
For example, suppose that a citizen presents at a doctors surgery, and some test suggest the citizen is infected with influenza, 
which supports the theory that there is an influenza epidemic (i.e. we observe $\expect_\mathit{citizen}\mathit{Flu}$). 
However, the observation is not of a random person, but rather of a person who elected to attend a doctors surgery, 
and therefore the actual observation is more likely the likelihood the citizen has flu, given that they presented at the surgery:
$\cond{\mathit{Flu}}{\mathit{sympt}}{\mathit{citizen}}$.

\subsection{Concept learning}
Role learning is a natural application of Bayes' law since the learning is applied to 
a probability distribution of possible individuals.
However, the probabilities of atomic concepts are modelled as dice, and hence independent of all other 
variables beyond the possible individual. This means we gain no additional information from applying Bayes' law.
If it was possible to observe atomic concepts directly (and often) 
it would be simple to refine a statistical model of the probabilities. 
Observations in \ADL\ are complex formulas, so it is preferable to find a more general solution.  

Concept learning addresses these issues by introducing new possible individuals in such a way that they do not affect any 
expected values for named individuals but with variations in the aleatoric probability of concepts, 
which may then be learnt via role learning, given arbitrary observations. 

The key to this update is the following observation. 
Suppose that you are playing a game where you must roll a one on a biased die, and you do not know the bias.
The game master offers you a choice you can either roll once as usual, or alternatively you may flip a fair coin.
If the coin lands heads you have two chances to roll a one, but if the coin lands tails, you must roll two ones in a row.
Both these choices have exactly the same chance of success. 
In the first instance you have a $p$ chance of winning, and in the second instance you have a $\frac{1-(1-p)^2+p^2}{2} = p$ 
chance of success. 
However, in the second instance there are two alternatives, one with an increased chance of success 
and one with a diminished chance of success.
These two alternatives may then be subjected to role learning.

\begin{definition}\label{def:A-extension}
  Given a pointed aleatoric belief model $\interp_i = (I,r,\ell,i)$ and some concept to refine, $A\in X$,
  the $A${\em -extension of} $\interp_i$ is the aleatoric belief model $\interp_i^A = (I',r', \ell', i)$, 
  where: 
  \begin{enumerate}
    \item $I'=I\cup \{i^*\}$
    \item For all $j\in I$, for all $\rho\in\roles$, $r'(\rho)(j,i) = r'(\rho)(j,i^*) = r(\rho)(j,i)/2$, and 
      $r'(\rho)(i,j) = r'(\rho)(i^*,j) = \rho(i,j)$ where $j\neq i$.
    \item For all $j\in I$, for all $B\in\concepts$,  
      \begin{enumerate}
        \item $\ell'(j,B) =\ell(j,B)$, if $B\neq A$ or $j\neq i$,
        \item where $B\neq A$, $\ell'(i^*,B) = \ell(i,B)$, 
        \item $\ell'(i^*,A) = \ell(i,A)^2$, and
        \item $\ell'(i, A) = 2\ell(i,A)-\ell(i,A)^2$
      \end{enumerate}
  \end{enumerate}
\end{definition}

Note that in the $A$-extension of $\interp_i$, $i$ and $i^*$ will be related by $\id$, 
and are identical except for the probability assigned to $A$.
This intuition can be formalised via the probability spaces of Lemma~\ref{lem:ProbSpace}.
\begin{lemma}
  Let $\interp_i = (I,r,\ell,i)$ be a pointed aleatoric belief model, and let $A\in X$.
  Given $\mathcal{C} = \interp^A_i$ and $\mathcal{D} = \interp^A_{i'}$, and a set of $\ALC$ formulas, 
  $\mathcal{L}$ as in Definition~\ref{def:probMeasure}, let $(\Omega, \hat{\mathcal{L}},\mathcal{P})$ 
  be a probability space defined by $\Omega = \Omega^{\mathcal{C}}\cup\Omega^{\mathcal{D}}$, and
  for all $\hat{\alpha}\in\hat{\mathcal{L}}$, 
  $$\mathcal{P}(\hat{\alpha}) = \frac{\mathcal{P}^{\mathcal{C}}(\hat{\alpha})+\mathcal{P}^{\mathcal{D}}(\hat{\alpha})}{2}.$$
  Then $(\Omega, \hat{\mathcal{L}},\mathcal{P})$ is a probability space, and furthermore, 
  for all $\hat{\alpha}\in\hat{\mathcal{L}}$, $\mathcal{P}^{\interp_i}(\hat{\alpha}) = \mathcal{P}(\hat{\alpha})$.
\end{lemma}
This follows from the reasoning above, and it is straightforward to check the conditions of a probability space are met. 

In practice this operation can be applied many times to learn correlations between different concepts. 
However, performing these operations across all individuals and a set of $n$ atomic concepts, 
will lead to a $2^n$ factor increase in the size of the model so once the distribution of a concept has been refined 
the two possible individuals $i$ and $i^*$ can be combined back into a single individual by 
taking the sum of the probabilities of their concepts and roles, weighted by the learnt distribution for $\id$.

In the example of Subsection~\ref{sbsect:example}, 
suppose that the assessment that the likelihood of Hector having a fever is to be reassessed,
based on the observation (or possibly erroneous belief) 
that Hector would have a fever if and only if Hector's contact had a fever. 
A new world $\hO$ is replaced by $\hO^1$ and $\hO^2$ 
where $\hO^1(F) = 2\hO(F)-\hO(f)^2 = 0.84$, 
and $\hO^2(F) = \hO(F)^2 = 0.36$.
The probabilities are then updated via role learning over $\id$, 
given the observation $\phi = \expect_\id{\ite{F}{\expect_c{F}}{\expect_c\lnot F}}$, 
where the relevant fragment of the aleatoric belief model is shown in Figure~\ref{fig:concept-learning}.
Note, the model has been revised to make the example clearer.
\begin{figure}
  \begin{center}
    \scalebox{0.8}{
    \begin{tikzpicture}
      \draw (1,3) node[circle,draw](h0) {\tiny{ $\begin{array}{c}\hO\\{\mathit F}:0.6\end{array}$ }};
      \draw (3,2) node[circle,draw](j1) {\tiny{ $\begin{array}{c}\jI\\{\mathit F}:0.2\end{array}$ }};
      \draw (3,4) node[circle,draw](j0) {\tiny{ $\begin{array}{c}\jO\\{\mathit F}:0.9\end{array}$ }};
      \draw (h0) -- (j0)  node[pos=0.9,left] {0.8};
      \draw (h0) -- (j1)  node[pos=0.9,left] {0.2};

      \draw (6,5) node[circle,draw](h0n) {\tiny{ $\begin{array}{c}\hO\\{\mathit F}:0.84\end{array}$ }};
      \draw (6,1) node[circle,draw](h0s) {\tiny{ $\begin{array}{c}\hO^*\\{\mathit F}:0.36\end{array}$ }};
      \draw (8,2) node[circle,draw](j1n) {\tiny{ $\begin{array}{c}\hI\\{\mathit F}:0.2\end{array}$ }};
      \draw (8,4) node[circle,draw](j0n) {\tiny{ $\begin{array}{c}\jO\\{\mathit F}:0.9\end{array}$ }};
        \draw[dashed] (h0n) -- (h0s) node[midway, right](pt-h) {} node[near start, left] {0.5 (0.42)} node[near end,left] {0.5 (0.58)};
      \draw (pt-h) -- (j1n) node[near end,left] {0.2};
      \draw (pt-h) -- (j0n) node[near end,left] {0.8};
    \end{tikzpicture}
    }
  \end{center}
  \caption{Concept learning applied to the aleatoric belief model in Figure~\ref{fig:example}, 
    where Hector applies the belief that he would only have a fever if and only if a contact had a fever.
    The model on the left is the $F$-extension, and the probabilities in brackets are the values after role learning has been applied to $\id$.}\label{fig:concept-learning}
\end{figure}
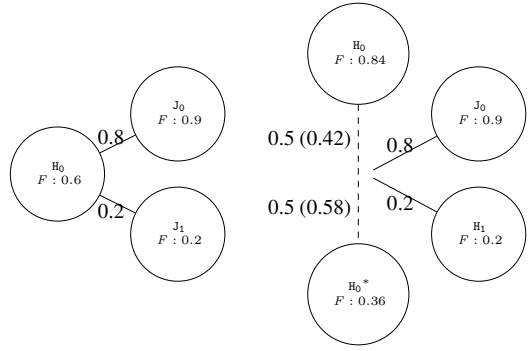
In the individuals $\hO$ and $\hO^*$, role learning can be applied over $\id$. 
The probability of $\phi$ in $\hO$ is approximately $0.4$ 
and the probability of $\phi$ in $\hO^*$ is approximately $0.55$, 
so after role learning has been applied to $\id$ 
the probability of $\hO$ is 0.42
and the probability of $\hO^*$ is 0.58.
Finally, aggregating $\hO$ and $\hO^*$ into a single node by 
taking weighted sums of the likelihoods gives the probability of $F$ to be 0.56.

This example also demonstrates how $\ADL$ can work with complex beliefs, such as $\phi$.

\section{Related work}\label{sec:relatedwork}

There is a substantial amount of work on logics for reasoning about uncertainty \cite{halpern:2003}, 
including \cite{kozen,kooi,vanBenthem, nilsson}, and going back to the works of Ramsey 
\cite{ramsey}, Carnap \cite{carnap} and de Finetti \cite{deFinetti}. 

Probabilistic modal logics have also been studied extensively \cite{halpern:2003,kozen,kooi}.
The approach to reasoning about probabilities used here is to have explicit probabilities in the syntax of the logic.
Therefore, these logics can express propositions such as {\em the chance of rain in 68\%}, or
{\em it is twice as likely to rain as it is to snow}.
These propositions are either true or false, so they reason about probabilities, rather than reasoning probabilistically.
Axioms and model-checking procedures have been provided for these logics, but they tend to be hard to apply in practise,
since the probabilities are explicit and hard to calculate.
By contrast, an aleatoric proposition would be something like {\em rain-today}, which could be evaluated as 0.68, 
but it is an implicit value that may vary with an agent's experience.

{\em Markov Logic Networks} \cite{mln} (generalising Bayesian networks and Markov networks)
address a similar problem of providing a logical interface to machine learning methods.
These approaches attach a probabilistic interpretation to formulas in a fragment of first order logic,
rather than providing a probabilistic variation of first order logic operators.
Therefore, whilst providing some of the benefits of logical approaches in a machine learning context,
there is only a weak coupling between first order deduction and the probabilistic semantics.

There is some commonality in purpose with {\em probabilistic logic programming} 
\cite{lukasiewicz1998probabilistic,de2007problog}. 
However, the concepts are constrained to be Horn clauses, 
where atomic formula are mutually independent.

There is a growing body of work addressing the need for probabilistic reasoning in knowledge bases. 
In \cite{heinsohn}, an inductive reasoning approach is applied to include probabilities with rules;
in \cite{lutz-DL}, a subjective Bayesian approach is proposed to describe 
the probabilities associated with a concept or role holding; and Lukasiewicz and Straccia \cite{lukasiewicz} 
have proposed a method to include vagueness (or fuzzy concepts \cite{zadeh}) in descriptions logics.
Probabilistic extensions of description logics have also been proposed by 
Rigguzzi et al \cite{riguzzi2015probabilistic} and Pozzato \cite{pozzato2019typicalities}.
These approaches extend knowledge bases to include probabilistic assertions and axioms, 
and provide an extended syntax for querying probability thresholds. 
Some work on learning parameters and structure of knowledge bases via probabilistic description logics has been done, 
including Ceylan and Penaloza \cite{ceylan}, who have proposed a Bayesian Description Logic that combines a 
basic description logic framework with Bayesian networks \cite{BayesNet} for representing uncertainty about facts,
and Ochoa Luna et al \cite{ochoa2011learning} who applied statistical methods to estimate the most likely configuration of a knowledge base.

These approaches are very different to the work presented here, 
as probabilities are not propagated through the roles, 
and they do not permit learning based on the observation of complex propositions.

\section{Conclusion}

This paper has introduced a novel approach for representing uncertain knowledge and beliefs. 
Generalising the description logic \ALC, the aleatoric description logic is able to 
represent complex concepts as independent aleatoric events. 
The events are contingent on {\em possible individuals} 
so they give a subjective Bayesian interpretation of knowledge bases.
This paper has also given computational reasoning methods for aleatoric knowledge bases, 
and shown how aleatoric description logic corresponds to a probability space of functional $\ALC$ models.
Importantly, the syntax of \ADL\ does not include explicit probabilities so these do not 
need to be known a priori, and can be learnt and integrated into the aleatoric belief set based on observations.
The aleatoric concepts and roles enable a simple learning framework where
agents are able to update their beliefs based on the observations of complex propositions.

\label{sect:bib}
\bibliographystyle{plain} 

\providecommand{\noopsort}[1]{}

%
%
%
%
%

\end{document}